%% file: tkde.tex
\documentclass[lettersize,journal]{IEEEtran}
\usepackage{amsmath,amsfonts}
\usepackage{amsthm} 
\usepackage{algorithm}
\usepackage{algpseudocode} 
\usepackage{etoolbox}

\usepackage{array}
\usepackage[caption=false,font=normalsize,labelfont=sf,textfont=sf]{subfig}
\usepackage{textcomp}
\usepackage{stfloats}
\usepackage{url}
\usepackage{verbatim}
\usepackage{graphicx}
\usepackage{cite}
\usepackage{booktabs}
\usepackage{bm}
\usepackage{tabularx}
\usepackage{multirow}
\usepackage{xcolor}
\usepackage{enumitem}
\newtheoremstyle{propositionstyle}{3pt}{3pt}{\itshape}{}{\bfseries}{.}{ }{}
\theoremstyle{propositionstyle}
\newtheorem{proposition}{Proposition}
\begin{document}
\newcommand{\blue}{\textcolor{black}}
\newcommand{\bluetable}{\color{black}}
\title{Sharpness-Aware Poisoning: Enhancing Transferability of Injective Attacks on Recommender Systems}

\author{Junsong Xie, Yonghui Yang, Pengyang Shao, Le Wu,~\IEEEmembership{Member,~IEEE}
\IEEEcompsocitemizethanks{
\IEEEcompsocthanksitem This work was supported in part by grants from the National Natural Science Foundation of China under Grant U23B2031 and Grant 62436003, and in part by grants from the Fundamental Research Funds for the Central Universities under Grant JZ2025HGPB0248 and Grant PA2025IISL0106. \textit{(Corresponding author: Le Wu.)}
    \IEEEcompsocthanksitem  Junsong Xie, Pengyang Shao, Le Wu are with the Key Laboratory of Knowledge Engineering with Big Data, Hefei University of Technology, Hefei 230009, China. Junsong Xie is also with the Intelligent Interconnected Systems Laboratory of Anhui Province, Hefei University of Technology, Hefei 230009, China (e-mail: jsxie.hfut@gmail.com; shaopymark@gmail.com; lewu.ustc@gmail.com).
    \IEEEcompsocthanksitem Yonghui Yang is with the Laboratory of NExT++, National University of Singapore, Singapore, and with the Intelligent Interconnected Systems Laboratory of Anhui Province, Hefei University of Technology, Hefei 230009, China (e-mail: yh\_yang@nus.edu.sg).

}}

\markboth{Journal of \LaTeX\ Class Files,~Vol.~14, No.~8, August~2021}%
{Shell \MakeLowercase{\textit{et al.}}: A Sample Article Using IEEEtran.cls for IEEE Journals}

\maketitle
\newcommand{\name}{\textit{SharpAP}}

\input{content/0-abs}

\input{content/1-intro}

\input{content/3-pre}
\input{content/4-method}
\input{content/5-experiment}
\input{content/2-related}

\input{content/6-conclusion}


\begin{small}
\bibliographystyle{abbrv}
\bibliography{ref}
\end{small}

\input{author}

\end{document}

%% file: content/0-abs.tex
\begin{abstract}
Recommender Systems~(RS) have been shown to be vulnerable to injective attacks, where attackers inject limited fake user profiles to promote the exposure of target items to real users for unethical gains (e.g., economic or political advantages).
Since attackers typically lack knowledge of the victim model deployed in the target RS, existing methods resort to using a fixed surrogate model to mimic the potential victim model.
Despite considerable progress, we argue that the assumption that \textit{poisoned data generated for the surrogate model can be used to attack other victim models} is wishful.
When there are significant structural discrepancies between the surrogate and victim models, the attack transferability inevitably suffers.
Intuitively, if we can identify the worst-case victim model and iteratively optimize the poisoning effect specifically against it, then the generated poisoned data would be better transferred to other victim models.
However, exactly identifying the worst-case victim model during the attack process is challenging due to the large space of victim models.
To this end, in this work, we propose a novel attack method called Sharpness-Aware Poisoning (\textit{SharpAP}).
Specifically, it employs the sharpness-aware minimization principle to seek the approximately worst-case victim model and optimizes the poisoned data specifically for this worst-case model.  
The poisoning attack with \name~is formulated as a min-max-min tri-level optimization problem.
By integrating \name~into the iterative process for attacks, our method can generate more robust poisoned data which is less sensitive to the shift of model structure, mitigating the overfitting to the surrogate model.
Comprehensive experimental comparisons on three real-world datasets demonstrate that \name~can significantly enhance the attack transferability.


\end{abstract}

%% file: content/1-intro.tex
\maketitle

\IEEEpeerreviewmaketitle

\ifCLASSOPTIONcompsoc
\IEEEraisesectionheading{\section{Introduction}\label{sec:intro}}
\else

\vspace{-0.2cm}

\section{Introduction} \label{sec:intro}
\fi
\IEEEPARstart{I}{n} the digital age, Recommender Systems (RS) have emerged as indispensable tools for mitigating information overload across diverse platforms~\cite{wule2022survey,wu2018collaborative,zhang2019deep,yang2023generative}, from e-commerce (e.g., Amazon, Taobao) to social media (e.g., Twitter, Xiaohongshu).
By leveraging user behaviors such as purchase history and content consumption patterns, RS infers users' latent preferences and then recommends potentially interesting items. 
However, their growing significance has also drawn attention from adversarial entities seeking to exploit their vulnerabilities~\cite{wang2024poisoning,revAdv,CLeaR}.
RS typically utilizes Collaborative Filtering (CF) to provide recommendations, meaning that users’ recommendations are influenced not only by their own historical behaviors but also by interactions from other users. 
While this principle enhances recommendation accuracy by leveraging collective user preferences, it also introduces security risks~\cite{uba_www,wang2025id,wzw02}. 
Specifically, the openness of RS and collaborative patterns create opportunities for \textit{injective attacks}~\cite{dada_nips,zhang2024improving,wang2025graph}.
\begin{figure}[t]
  \begin{center}
    \vspace{-0.50cm}
    \includegraphics[scale=0.73]{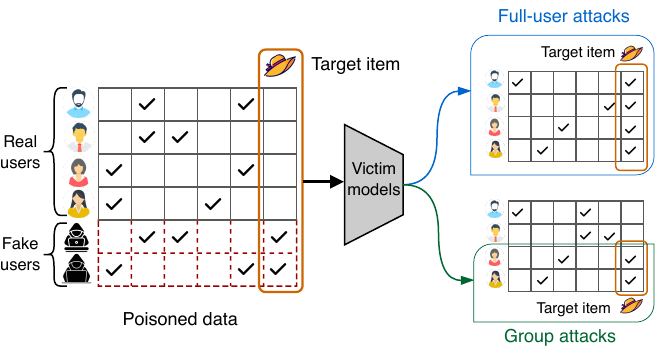}
  \end{center}
  \vspace{-0.40cm}
  \caption{Illustration of injective attacks on full-users and group users. Full-user attacks aim to increase the exposure probability of the target item among all real users, whereas group attacks focus on a specific group (e.g., the female user group).}
  \label{fig:intro_1}
  \vspace{-0.40cm}
\end{figure}

In injective attacks, attackers take advantage of the system’s openness by registering a limited number of fake user accounts and then designing their profiles elaborately~\cite{legup,dada_nips,aush}.
As illustrated in Fig.~\ref{fig:intro_1}, these fake user profiles, combined with real ones, form the \textit{poisoned data}.
This poisoned data is then used to attack the black-box victim model deployed in the target RS, manipulating it to produce recommendations that align with the attackers' goals~\cite{fang2020influence,incomplete_kdd21}.
For example, full-user attacks aim to promote the recommendation probability of the target item across all real users~\cite{wang2024poisoning,guo2023targeted}, whereas group attacks focus on a specific group~\cite{uba_www}.
To carry out injective attacks, existing works have proposed three main categories of attack methods:~heuristic-based, neural network-driven, and gradient-based attacks.
Heuristic-based attacks rely on manually crafted rules to design fake user profiles~\cite{burke2005limited,yang2017fake}.
Neural network-driven attacks optimize the parameters of neural networks to generate influential fake user behaviors to achieve the attack objective~\cite{fan2021attacking,wu2021triple}. 
Gradient-based attacks maximize the sophisticated attack objective by optimizing parameterized fake profiles via gradients~\cite{fang2020influence,dada_nips,incomplete_kdd21}.

\begin{small}
\begin{figure*} [t]
  \begin{center}
  \includegraphics[scale=0.83]{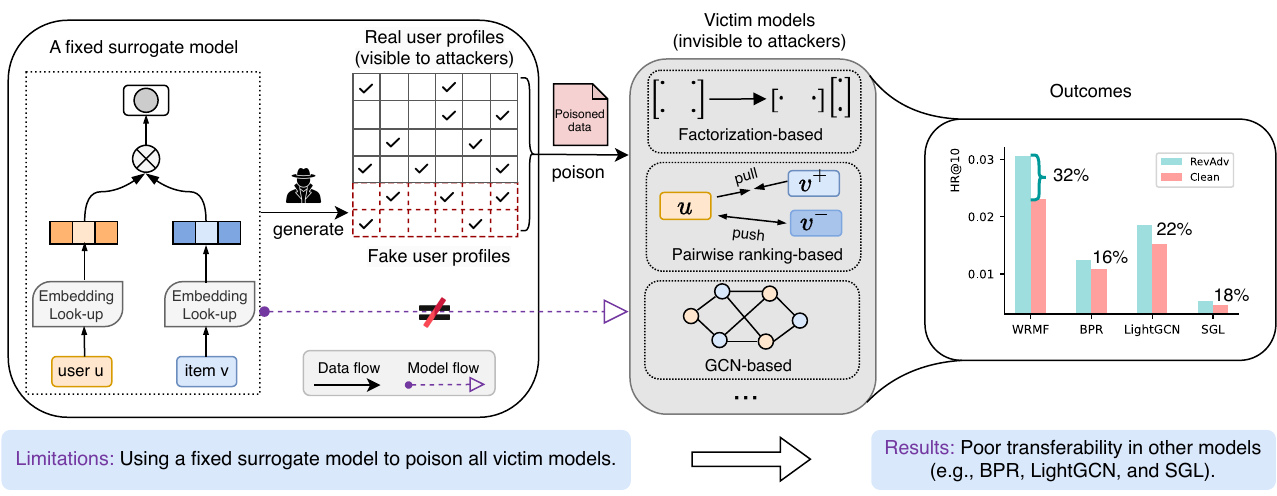}
  \end{center}
  \caption{
  Illustration of the origin of our work’s motivation.
  Since the victim model is inaccessible to attackers, existing methods use a fixed surrogate model to mimic it for generating poisoned data.
  These methods first ensure the effectiveness of poisoned data on the surrogate model.
Then, they hope it will be equally effective on all unseen victim models.
We argue that these methods neglect structural discrepancies between the surrogate and victim models, which inevitably leads to impaired attack transferability.
The experimental results presented on the right highlight the limitations.
Specifically, the RevAdv attack~\cite{revAdv} exhibits a good attack performance on the surrogate WRMF~\cite{wrmf}.
However, when the victim models are BPR~\cite{bpr2012}, LightGCN~\cite{lightgcn_he}, and SGL~\cite{sgl}, the performance gains from the attack are significantly diminished, indicating poor transferability. ``Clean'' refers to the results on the unattacked dataset.} 
  \label{fig:intro_2}
  \vspace{-0.40cm}
\end{figure*}
\end{small} 

Gradient-based attacks have emerged as the predominant methods for injective attacks, owing to their demonstrated performance advantages~\cite{dada_nips,cheng2024towards}.
As illustrated in Fig.~\ref{fig:intro_2}, the victim model deployed in the target RS is typically invisible to attackers, making it challenging for gradient-based attacks to obtain its gradients directly.
To circumvent this limitation, they use a fixed surrogate model to mimic the potential victim model.
The generated poisoned data is then expected to be effective against various unseen victim models.
This ``one poisons all'' strategy, however, is crucially dependent on a precarious assumption:~\textit {Poisoned data generated for the surrogate model can be used to attack other victim models}. 
Existing methods relying on this greedy assumption neglect structural discrepancies between the surrogate and victim models~\cite{revAdv,dada_nips,incomplete_kdd21}.
This oversight fundamentally limits the transferability of their attacks.
Our empirical validation in Fig.~\ref{fig:intro_2} also substantiates this limitation.
Specifically, we use the representative gradient-based method RevAdv~\cite{revAdv} as an example to inject fake users into a popular dataset Movielens-1M~\cite{movielens1}.
RevAdv employs WRMF~\cite{wrmf} as the surrogate model.
The objective is to increase the recommendation probability of five randomly selected target items across all real users. 
To evaluate the effectiveness of the attack, we adopt the Hit Ratio at 10 (HR@10) as the evaluation metric.
Since we cannot control the retraining process of the poisoned data, it may be used in various victim models, such as the surrogate WRMF~\cite{wrmf}, the pairwise ranking-based BPR~\cite{bpr2012}, and the GCN-based models (LightGCN~\cite{lightgcn_he} and SGL~\cite{sgl}).
Baseline performance (denoted as ``Clean'') represents the results on the unattacked dataset.
While effective against WRMF ($\uparrow$32\% compared to ``Clean''), the poisoned data exhibits poor transferability in BPR ($\uparrow$16\%), LightGCN ($\uparrow$22\%), and SGL ($\uparrow$18\%).
Fig.~\ref{fig:intro_2} verifies that relying solely on the inherent transferability of poisoned data is precarious and exhibits overfitting to the surrogate model. 
We should explicitly accommodate model structure shifts to enhance the transferability.

Intuitively, if we can iteratively optimize the poisoning effect against the \textbf{worst-case model} (the victim RS model that has the worst poisoning effect) rather than a fixed surrogate model, the generated poisoned data would be better transferred to other victim models~\cite{he2024sharpness}.
However, exactly identifying the worst-case model among a large space of potential victim models during the attack process is computationally intractable.
Meanwhile, sharpness-aware minimization has demonstrated notable success in enhancing model generalization by seeking the maximum loss within a local neighborhood of model parameters and then minimizing it~\cite{sharpness_mini,chen2023does, he2024sharpness}.
Motivated by this insight, we propose a novel attack method called \textbf{Sharp}ness-\textbf{A}ware \textbf{P}oisoning (\name), which seeks the approximately worst-case model via sharpness-aware minimization principle.
Specifically, \name~reformulates the attack process by extending the original min-min bi-level optimization into a min-max-min tri-level optimization problem.
The introduced maximization step performs a bounded perturbation to seek the worst-case model, leveraging the sharpness-aware minimization principle.
Distinct from existing methods that generate poisoned data for a fixed surrogate model~(i.e., without the maximization step), \name~dynamically targets the worst-case victim model, thereby mitigating the overfitting to the surrogate model.
Crucially, based on the characteristics of RS models, we provide a theoretical analysis of the transferability of sharpness-aware attacks.
In summary, our key contributions are as follows:
\begin{itemize}[leftmargin=*,topsep=0pt,parsep=0pt]
    \item To the best of our knowledge, we are the first to point out the challenge of overfitting to the surrogate model in injective attacks. 
    Our findings reveal that when the victim model’s structure shifts, the transferability of the poisoning effect inevitably suffers.
    
    \item 
    To overcome this challenge, we propose \name, which introduces sharpness-aware minimization principle to approximate the worst-case model.
    We then generate poisoned data for this worst-case model instead of a fixed surrogate model.
    Furthermore, we provide a theoretical analysis demonstrating that optimizing poisoned data against worst-case models can improve attack transferability.
    \item 
    Technically, \name~formulates injective attacks as a sharpness-aware min-max-min tri-level optimization problem. Moreover, \name~can be seamlessly integrated with many existing gray-box attack methods to enhance the transferability.
    
    \item We conduct comprehensive experimental studies and demonstrate that \name~significantly enhances the transferability of many representative victim models.
\end{itemize}

%% file: content/3-pre.tex
\begin{table}[!t]
	\centering
	\caption{Mathematical Notations.}
	\setlength\tabcolsep{8pt}{
	\begin{tabular}{cl}
		\toprule
		\textbf{Notation} & \textbf{Description}\\
		\midrule
            ${u}^r$, ${u}^f$ & Real user and fake user. \\ 
		${v}$, ${v}^t$ & Item and target item. \\ 
           
            \midrule
		$U^r$, $U^f$, $U$ & Real user set, fake user set and all user set. \\ 
		$V$, $V^t$ & Item set and target item set. \\
		
	\midrule
	$\mathbf{R}^r$, $\mathbf{R}^f$, $\mathbf{R}$ & Real user data, fake user data, and all user data. \\ 
         $\hat{\mathbf{R}}$ & Predicted scores. \\ 
        \midrule
	$\mathcal{M}^s$, $\mathcal{M}^v$ & Surrogate model and victim model. \\ 
    $\mathcal{M}^{worst}$ & Worst-case model. \\
    \midrule
    $\Omega$ & Victim model space.\\
        \midrule
	$\theta^*$ & Optimal parameter of the surrogate model. \\ 
        $\theta^{\Delta}$ & A bounded perturbation to $\theta^*$. \\
	\midrule
	$\mathcal{L}_{rec}$ & Recommendation loss. \\ 
    $\mathcal{L}_{atk}$ & Attack objective loss. \\
	\bottomrule
	\end{tabular}
	\label{tab::symbol}
	}
    \vspace{-0.30cm}
\end{table}
\section{Preliminaries}
\subsection{Recommender System under Injective Attack}
In a recommender system under injective attack, there are three entity sets: a real user set $U^r=\{u^r_1,u^r_2,...,u^r_{|U^r|}\}$, a fake user set $U^f=\{u^f_1,u^f_2,...,u^f_{|U^f|}\}$ and an item set $V=\{v_1,v_2,...,v_{|V|}\}$. 
User-item interactions are recorded as feedback data, categorized into explicit feedback (e.g., direct ratings like 5-star scores) and implicit feedback (e.g., indirect signals such as purchases or views). Due to the prevalence of implicit feedback in real-world applications, we focus on this type in our work. 
To represent user-item interactions, we define a binary matrix $\mathbf{R}=\begin{bmatrix}
\mathbf{R}^{r} \\
\mathbf{R}^{f}
\end{bmatrix} \in \{0,1\}^{{(|U^r|+|U^f|)}\times {|V|}}$, where $r_{uv}=1$ indicates positive feedback from user $u$ on item $v$, and $r_{uv}=0$ denotes an unknown interaction.
RS leverages historical user-item interactions to predict a relevance score matrix $\hat{\mathbf{R}}\in \mathbb{R}^{{(|U^r|+|U^f|)}\times {|V|}}$, where higher scores indicate stronger relevance between a user and an item. The objective of RS is to rank items for each user based on these predicted relevance scores, ensuring that the most relevant items are prioritized.
\subsection{Injective Attacks on Recommender System}
In this section, we first present the formal definition of injective attacks.
Then, we revisit the strategy employed by gradient-based attack methods under a gray-box setting. 
Given the real data $\mathbf{R}^r \in \{0,1\}^{{|U^r|}\times {|V|}}$, the target RS (a.k.a. victim model) $\mathcal{M}^v$, and a limited fake users $U^f=\{u^f_1,u^f_2,...,u^f_{|U^f|}\}$, the corresponding fake data $\mathbf{R}^f \in \{0,1\}^{{|U^f|}\times {|V|}}$ can be learned to minimize the attack loss function $\mathcal{L}_{atk}$ as follows:
\begin{gather}
    \min_{\mathbf{R}^f} \mathcal{L}_{atk}(\mathcal{M}^v(\theta^*;\mathbf{R}^r)), \label{eq:bi_11}\\
    \text{s.t.} \quad \theta^* = \arg\min_{\theta} ( \mathcal{L}_{rec}(\mathbf{R}, \mathcal{M}^v(\theta;\mathbf{R})),
\label{eq:bi_12}
\end{gather}
where
$\mathbf{R}=\begin{bmatrix}
\mathbf{R}^{r} \\
\mathbf{R}^{f}
\end{bmatrix}$.
$\mathcal{L}_{rec}$ is the recommendation training loss.
$\mathcal{M}^v(\theta;\mathbf{R})$ and $\mathcal{M}^v(\theta^*;\mathbf{R}^r)$ represent the predictions of the victim model $\mathcal{M}^v$ for all users and real users, respectively, given parameters $\theta$ and $\theta^*$.
For clarity, we omit the attacker’s capability constraints, $|U^f|\le \delta|U^r|$ and $||\mathbf{R}^f[u^f]||_0 \le N$, to keep the formula concise.
$\delta$ represents the proportion of fake users to real users, and $N$ denotes the maximum number of items each fake user can interact with (i.e., profile size).
It is a non-trivial task to generate the exact optimal fake data $\mathbf{R}^f$ that achieves the minimum attack loss, owing to the exponential candidate search space $\mathcal{O}({|V|}^{N\cdot|U^f|})$.
To obtain an approximate optimal solution, there are three types of attack methods:~heuristic-based, neural network-driven, and gradient-based attacks.
In this paper, we focus on gradient-based attacks due to their demonstrated effectiveness in prior studies.
We perform attacks under a gray-box setting, where attackers have access only to the real data $\mathbf{R}^r$, but no knowledge of the victim model $\mathcal{M}^v$.
This setting is reasonable, as proprietary recommendation algorithms are rarely disclosed by platforms. Nevertheless, attackers can obtain user behavior data through open APIs or by scraping public user profiles~\cite{dada_nips,uba_www}.
Under a gray-box setting, existing gradient-based attack methods use a fixed surrogate model $\mathcal{M}^s$ to replace the inaccessible victim model $\mathcal{M}^v$ in Eq.\eqref{eq:bi_11} and \eqref{eq:bi_12}.
After the replacement, the inner optimization (i.e., Eq.\eqref{eq:bi_12}) aims to mimic the retraining process with the poisoned data $\mathbf{R}$ using the surrogate model $\mathcal{M}^s$.
The outer optimization (i.e., Eq.\eqref{eq:bi_11}) is responsible for optimizing the fake data $\mathbf{R}^f$ to fulfill the attack goals.
Fig.~\ref{fig:flow} shows the flow chart of fake profiles instantiated.
Although solving such a bi-level optimization problem ensures satisfactory attack performance on the surrogate model, the transferability of the attack to other victim models remains uncertain and largely uncontrolled.

\begin{figure}[t]
  \begin{center}
    \includegraphics[scale=0.50]{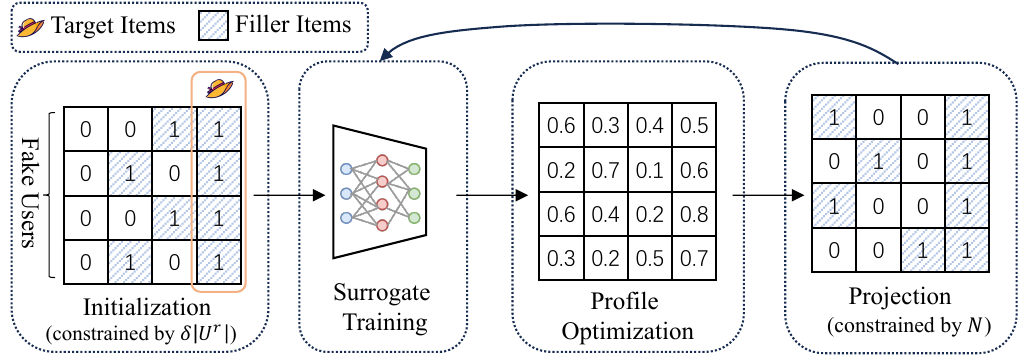}
  \end{center}
  \vspace{-0.40cm}
 \caption{A flow chart of fake profiles instantiated.}
  \label{fig:flow}
  \vspace{-0.40cm}
\end{figure}

%% file: content/4-method.tex
\begin{small}
\begin{figure*} [t]
  \begin{center}
  \includegraphics[scale=0.87]{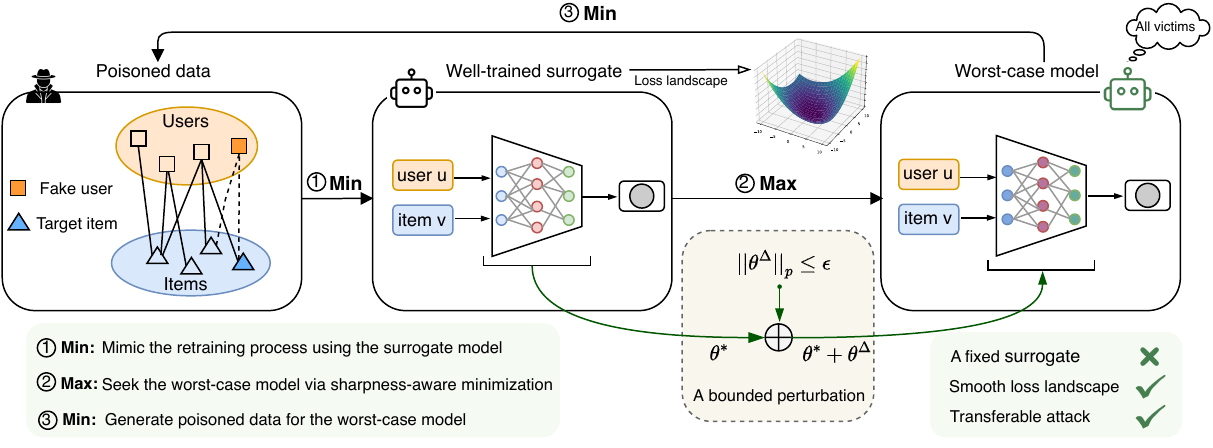}
  \end{center}
    \vspace{-0.30cm}
  \caption{The overall illustration of our method \name. We propose a sharpness-aware tri-level optimization, which seeks the worst-case model (i.e., the victim model with the worst poisoning effect) under a bounded perturbation to generate robust poisoned data.}
  \label{fig:method_8}
  \vspace{-0.30cm}
\end{figure*}
\end{small}

\section{Methodology}
\label{sec:method}

\subsection{The Objective of \name}
We begin by theoretically defining the victim model space, denoted as $\Omega$. This space encompasses all potential victim models whose objective is to minimize the recommendation training loss $\mathcal{L}_{rec}$ on the poisoned data $\mathbf{R}$:
\begin{equation}
    \Omega = \{ \mathcal{M}^{{v}} | \mathcal{L}_{rec}(\mathbf{R},\mathcal{M}^{v}(\mathbf{R})) \leq \rho \},
    \label{eq:victim_space}
\end{equation}
where $\mathcal{M}^v(\mathbf{R})$ represents the predictions of the victim model $\mathcal{M}^v$ for all users.
Mathematically, the victim model space $\Omega$ is defined as the set of all models such that the recommendation loss does not exceed a predefined threshold $\rho$. 
This formulation ensures that $\Omega$ includes only those models that are well-trained on the poisoned data.
Then, we identify the worst-case model $\mathcal{M}^{{worst}}$ in the space $\Omega$ by maximizing the attack loss $\mathcal{L}_{atk}$. The worst-case model $\mathcal{M}^{{worst}}$ can be formulated as:
\begin{equation}
    \mathcal{M}^{worst} = \arg\max_{\mathcal{M}^{v} \in \Omega} \mathcal{L}_{atk}(\mathcal{M}^{v}(\mathbf{R}^r)).
    \label{eq:find_worst}
\end{equation}
According to Eq.\eqref{eq:find_worst}, we can obtain:
\begin{equation}
    \mathbb{E} \left[ \mathcal{L}_{atk}(\mathcal{M}^{v}(\mathbf{R}^r)) \right] \leq \mathcal{L}_{atk}(\mathcal{M}^{worst}(\mathbf{R}^r)).
    \label{eq:expect3}
\end{equation}
Eq.\eqref{eq:expect3} states that $\mathcal{L}_{atk} (\mathcal{M}^{worst}(\mathbf{R}^r))$ is the upper bound of the expected attack loss over all victim models. By optimizing the fake data $\mathbf{R}^f$ to minimize this upper bound through:
\begin{equation}
    \min_{\mathbf{R}^f} \mathcal{L}_{atk}(\mathcal{M}^{worst}(\mathbf{R}^r)),
\end{equation}
we aim to reduce the average attack loss across all victim models. 
In other words, improving the poisoning effect on the worst-case model $\mathcal{M}^{worst}$ facilitates the propagation of the poisoning effect to other victim models, thereby enhancing the attack transferability.

However, directly solving Eq.~\eqref{eq:victim_space} and \eqref{eq:find_worst} is computationally intractable, and it is challenging to obtain an exact solution.
To address this, we leverage sharpness-aware minimization principles to approximate the worst-case solution~\cite{sharpness_mini,chen2023does,he2024sharpness}.
Specifically, we approximate the worst-case model by performing a localized search in the neighborhood of the surrogate model, formulated as:
\begin{equation}
\begin{gathered}
    \mathcal{L}_{atk}(\mathcal{M}^{worst}(\mathbf{R}^r)) \approx \max_{||\theta^\Delta||_p\leq \epsilon} \mathcal{L}_{atk}(\mathcal{M}^s(\theta^*+\theta^\Delta;\mathbf{R}^r)), \\
    \text{s.t.}~\theta^* = \arg\min_{\theta} ( \mathcal{L}_{rec}(\mathbf{R}, \mathcal{M}^s(\theta;\mathbf{R}))).
\end{gathered}
\label{eq:bi_2}
\end{equation}
As formulated in Eq.\eqref{eq:bi_2}, we first train the surrogate model $\mathcal{M}^s$ on the poisoned data $\mathbf{R}=\begin{bmatrix} \mathbf{R}^{r} \\ \mathbf{R}^{f} \end{bmatrix}$ using a recommendation loss function, ensuring that the optimized parameter $\theta^*$ satisfies the constraint $\mathcal{L}_{rec}(\mathbf{R},\mathcal{M}^s(\theta^*;\mathbf{R})) \leq \rho$.
We then introduce a bounded perturbation $\theta^\Delta$ to $\theta^*$, constrained by $\|\theta^\Delta\|_p \leq \epsilon$, and maximize the attack loss $\mathcal{L}_{atk}$ over the perturbed $\theta^* + \theta^\Delta$ to seek the worst-case model.  
This perturbation explores local regions around $\theta^*$, which are presumed to lie within $\Omega$. 
Consequently, the perturbed model exhibits higher attack loss than $\theta^*$, mimicking the worst-case behaviour while remaining within the plausible victim model space.
By replacing the fixed surrogate model $\mathcal{M}^s$ with the worst-case model $\mathcal{M}^{worst}$, we formulate the \textit{\textbf{Sharpness-aware Poisoning}} (\name) problem as follows:
\begin{gather}
\overbrace{\min_{\mathbf{R}^f}\underbrace{\max_{||\theta^\Delta||_p\leq \epsilon} \mathcal{L}_{atk}(\mathcal{M}^s(\theta^*+\theta^\Delta;\mathbf{R}^r))}_\text{seek the worst-case model}}^\text{sharpness-aware worst-case~optimization}, 
     \label{eq:bi_31}\\
   \text{s.t.}~\theta^* = \arg\min_{\theta} ( \mathcal{L}_{rec}(\mathbf{R}, \mathcal{M}^s(\theta;\mathbf{R}))).\label{eq:bi_32}
\end{gather}
Formally, \name~extends the bi-level optimization in Eq.\eqref{eq:bi_11} and \eqref{eq:bi_12} to a sharpness-aware tri-level optimization. The term in the outer optimization can be defined as the \textit{sharpness-aware attack objective}~(i.e., $\max_{||\theta^\Delta||_p\leq \epsilon} \mathcal{L}_{atk}(\mathcal{M}^s(\theta^*+\theta^\Delta;\mathbf{R}^r))$).
\textit{SharpAP} improves existing methods by replacing their attack objective with a sharpness-aware counterpart. 
By iteratively optimizing the poisoning effect against the worst-case model~(i.e., sharpness-aware worst-case optimization), \name~explicitly treats attack transferability as a learning objective. Consequently, it mitigates the overfitting issue and enhances the transferability.
The detailed tri-level optimization process is shown in Fig.~\ref{fig:method_8}.

\begin{algorithm}[tb]
    \caption{Sharpness-aware poisoning attack (\name)}
    \label{alg:algorithm}
    \textbf{Input}: Real user data $\textbf{R}^r$; learning rate for inner and outer objective: $\lambda_1$ and $\lambda_2$; max iteration for inner and outer objective: T and L.\\
    \textbf{Parameter}: Parameter $\theta$ for the surrogate model $\mathcal{M}^s$ \\
    \textbf{Output}:Fake user data $\mathbf{R}^f$ 
    \begin{algorithmic}[1] 
    \State Initialize fake user data $\mathbf{R}^f$ and the parameter $\theta$ of the surrogate model $\mathcal{M}^s$
         \For{$l=1$ to $L$}
        \Statex \hspace*{0.17in}  \textbf{// (1) Surrogate model optimization} 
        \For{$t=1$ to $T$}
            \State Optimize surrogate model parameter with SGD:
\Statex \hspace{\algorithmicindent}$\theta^{(t)} \leftarrow \theta^{(t-1)}-\lambda_1\nabla_{\theta^{(t-1)}}(\mathcal{L}_{rec}(\mathbf{R},\mathcal{M}^s(\theta^{(t-1)},\mathbf{R})))$
        \EndFor
        \State $\theta^*\leftarrow\theta^{(T)}$
        \Statex \hspace*{0.17in} \textbf{// (2) \name}
        \State Seek the worst-case model according to Eq.\eqref{eq:bi_2}
        \State Solve ${\theta}^\Delta$ according to Eq.\eqref{eq:worst_theta2}
        \State Compute gradients $\nabla_{\mathbf{R}^f}\mathcal{L}_{atk}$ according to Eq.\eqref{eq:final_approx}
        \Statex \hspace*{0.17in} \textbf{// (3) Projected gradient descent}
        
        \State Update fake user data: $\mathbf{R}^f=Proj(\mathbf{R}^{f}-\lambda_2\nabla_{\mathbf{R}^f}\mathcal{L}_{atk})$
        \EndFor
        \State \textbf{return} $\mathbf{R}^f$ 
    \end{algorithmic}
\end{algorithm}

\subsection{The Optimization of \name}
Compared to existing attack objectives, sharpness-aware
attack objective introduces an additional worst-case perturbation ${\theta}^\Delta$ to the surrogate parameter $\theta^*$. 

Therefore, we first need to solve for $\hat{\theta}^\Delta$.
Specifically, we follow the approach in \cite{sharpness_mini} to approximate the maximization problem via a first-order Taylor expansion, as follows:
\begin{equation}
\begin{aligned}
\hat{\theta}^\Delta = \epsilon &\cdot \text{sign} \left( \nabla_\theta \mathcal{L}_{atk}(\mathcal{M}^s(\theta^*;\mathbf{R}^r)) \right) \left| \nabla_\theta \mathcal{L}_{atk}(\mathcal{M}^s(\theta^*;\mathbf{R}^r)) \right|^{q-1} \\
& \cdot \left( \| \nabla_\theta \mathcal{L}_{atk}(\mathcal{M}^s(\theta^*;\mathbf{R}^r)) \|^q_q \right)^{1/p},
    \label{eq:worst_theta}
\end{aligned}
\end{equation}
where $1/p+1/q=1$, and we set $p=2$ as \cite{sharpness_mini}.
The computation of Eq.\eqref{eq:worst_theta} can be further simplified as:
\begin{equation}
\begin{aligned}
\hat{\theta}^\Delta = \epsilon &\cdot \nabla_\theta \mathcal{L}_{atk}(\mathcal{M}^s(\theta^*;\mathbf{R}^r)) 
  \left( \| \nabla_\theta \mathcal{L}_{atk}(\mathcal{M}^s(\theta^*;\mathbf{R}^r)) \|^2_2 \right)^{1/2}.
    \label{eq:worst_theta2}
\end{aligned}
\end{equation}
Then, we can have the approximation to calculate the worst attack effect of the parameter $\theta^*$ under a bounded perturbation via replacing $\theta^*$ with $\theta^*+\hat{\theta}^\Delta$, as follows:
\begin{equation}
\begin{gathered}
\nabla_{\mathbf{R}^f}\max_{||\theta^\Delta||_p\leq \epsilon} \mathcal{L}_{atk}(\mathcal{M}^s(\theta^*+\theta^\Delta;\mathbf{R}^r)) \\
\approx \nabla_{\mathbf{R}^f}\mathcal{L}_{atk}(\mathcal{M}^s(\theta;\mathbf{R}^r))|_{\theta=\theta^*+\hat{\theta}^\Delta}.
    \label{eq:final_approx}
\end{gathered}
\end{equation}
Since the fake data $\mathbf{R}^f$ to be generated are discrete and subject to the profile size $N$ constraint, we introduce constrained gradient projection descent, as follows:
\begin{equation}
\begin{gathered}
\mathbf{R}^f=Proj(\mathbf{R}^{f}-\lambda_2\nabla_{\mathbf{R}^f}\mathcal{L}_{atk}(\mathcal{M}^s(\theta;\mathbf{R}^r))|_{\theta=\theta^*+\hat{\theta}^\Delta}),
    \label{eq:pgd2} \\
    \text{where}~Proj(\hat{r}_{u^fv}) =
\begin{cases} 
1, &   v \in N_{u^f}, \\
0, & otherwise,
\end{cases}
\end{gathered}
\end{equation}
where $\hat{r}_{u^fv}$ is the predicted score that fake user $u^f$ gives to item $v$ and 
$N_{u^f}$ is Top-N largest item set of fake user $u^f$ according to the score $\hat{r}_{u^fv}$. 
$N$ is the maximum number of items each fake user $u^f \in U^f$ can interact with. 

We present detailed \name~in Algorithm \ref{alg:algorithm}.
Compared to existing gradient-based methods~\cite{revAdv,incomplete_kdd21,dada_nips}, our method only introduces the calculation of a worst-case perturbation for the surrogate model. 
This calculation does not require additional attack knowledge or capabilities, and the computational cost is almost negligible~\cite{sharpness_mini}.

\subsection{\name~under Full-user and Group Attacks} \label{sub:full_and_group}
By replacing the attack loss function $\mathcal{L}_{atk}$ in Eq.\eqref{eq:bi_31}, different attack methods targeting specific goals (e.g., full-user attacks or group attacks) can be realized~\cite{dada_nips,li2016data, incomplete_kdd21,revAdv}.
Here, we use RevAdv~\cite{revAdv} as the backbone to implement the proposed \name~under full-user and group attacks as an example.
The goal of full-user attacks is to increase the recommendation probability of target items $v^t \in V^t=\{v^t_1,v^t_2,...,v^t_{|V^t|}\}$ for all real users in $U^r$. 
The attack loss function $\mathcal{L}_{atk}$ can be formulated as: 
\begin{equation}
\mathcal{L}^{full}_{atk}(\mathbf{\hat{R}}^r) = - \sum_{v^t \in V^t}\sum_{u^r \in U^r} \log \left( \frac{\exp(\hat{r}_{u^rv^t})}{\sum_{v \in {V}} \exp(\hat{r}_{u^rv})} \right).
    \label{eq:full_obj}
\end{equation}
The goal of group attacks is to target a specific group while minimizing the impact on other groups. 
We categorize real users in $U^r$ into two groups, $U_0$ and $U_1$, based on their binary attribute values. $U_0$  represents the set of users with an attribute value of 0, while  $U_1$ represents the set of users with an attribute value of 1.
Here, we attack the group $U_0$ as an example.

\begin{equation}
\begin{aligned}
\mathcal{L}^{group}_{atk}(\mathbf{\hat{R}}^r) = \sum_{v^t \in V^t}
(&\frac{1}{|U_1|} \sum_{u \in U_1} {\mathbb{I}_{v^t\notin \Gamma_u}} \hat{r}_{uv^t} - \frac{1}{|U_0|} \sum_{u \in U_0} {} \hat{r}_{uv^t}),
    \label{eq:group_obj}
    \end{aligned}
\end{equation}
where $\Gamma_u$ is Top-K ranked items for user $u$, $\mathbb{I}$ is an indicator function, if target item $v^t$ is not in the set $\Gamma_u$, then $\mathbb{I}=1$, otherwise $\mathbb{I}=0$. We can replace $\mathcal{L}_{atk}$ in Eq.\eqref{eq:bi_31} with Eq.\eqref{eq:full_obj} and Eq.\eqref{eq:group_obj} to implement sharpness-aware full-user attacks and group attacks, respectively.
For $\mathcal{L}_{rec}$, we employ a widely used weighted mean squared error loss function~\cite{revAdv,dada_nips} associated with implicit feedback matrix factorization. The specific formulation is:
\begin{equation}
    \mathcal{L}_{rec} = \sum_{u, v} c_{uv} (r_{uv} - \mathbf{u}_u^T \mathbf{v}_v)^2 + \lambda \left(\sum_u \|\mathbf{u}_u\|^2 + \sum_v \|\mathbf{v}_v\|^2\right),
\end{equation}
where $c_{uv}$ is the instance weight to differentiate observed and missing interactions. $\mathbf{u}_u$ and $\mathbf{v}_v$ are the user and item latent vectors, respectively. $\lambda$ is the regularization parameter.
The hyperparameter settings in the $\mathcal{L}_{rec}$ loss are kept consistent with those in~\cite{revAdv} and~\cite{dada_nips} for fair comparison.
\vspace{-0.1cm}
\subsection{Theoretical Analysis}
Representative RS models~\cite{wrmf,bpr2012,lightgcn_he,simgcl} typically rely on learnable user and item embeddings to encode latent preferences. Despite differences in loss functions or aggregation mechanisms, these models map the same user–item interaction data into a common embedding space.
This shared embedding-based representation implies that retraining different RS models on the same data leads to parameter solutions that may not be arbitrarily distant, but instead tend to reside in a relatively constrained region of the embedding space.
Therefore, it is reasonable to assume that the optimal parameters (embeddings) of a victim model $\theta^v$ lie within a bounded neighborhood of the surrogate model parameters (embeddings) $\theta^*$: $||\theta^v - \theta^*||_2 \le \epsilon$.
\begin{proposition}[Transferability Bound via Sharpness-Aware Minimization]
\label{prop:2}
Let $\theta^*$ and $\theta^v$ denote the parameters of the surrogate model and the unknown victim model, respectively. We assume the victim model resides within an $\epsilon$-neighborhood of the surrogate model, i.e., $\|\theta^v - \theta^*\|_2 \le \epsilon$ with $\|\zeta\| \leq \epsilon$. Furthermore, assume the attack loss function $\mathcal{L}_{atk}(\theta)$ is $L$-Lipschitz smooth, implying $\|\nabla\mathcal{L}_{atk}(\theta_1) - \nabla\mathcal{L}_{atk}(\theta_2)\|_2 \leq L\|\theta_1 - \theta_2\|_2$. Under these conditions, the attack loss on the victim model is upper-bounded by the surrogate loss plus a sharpness-related term:
\begin{equation}
    \label{eq:span1}
   \mathcal{L}_{atk}(\theta^v) \leq \underbrace{\mathcal{L}_{atk}(\theta^*)}_{\text{Surrogate Performance}} + \underbrace{\epsilon \|\nabla \mathcal{L}_{atk}(\theta^*)\|_2 }_{\text{Local Sharpness}}+ \frac{L\epsilon^2}{2}.
\end{equation}
Furthermore, the maximization step in \name\ serves as a proxy for minimizing this upper bound.
\end{proposition}
\begin{proof}
    Since $\mathcal{L}_{atk}$ is $L$-smooth, satisfying the quadratic upper bound property, we perform a Taylor expansion around $\theta^*$:
    \begin{equation}
        \mathcal{L}_{atk}(\theta^v) \le \mathcal{L}_{atk}(\theta^*) + \nabla\mathcal{L}_{atk}(\theta^*)^\top (\theta^v - \theta^*) + \frac{L}{2}\|\theta^v - \theta^*\|^2_2 .
    \end{equation}
    Let $\zeta = \theta^v - \theta^*$. Since the victim model lies within the $\epsilon$-neighborhood, we have $\|\zeta\|_2 \le \epsilon$. By applying the Cauchy-Schwarz inequality to the first-order term, we obtain:
    \begin{equation}
        \nabla\mathcal{L}_{atk}(\theta^*)^\top \zeta \le \|\nabla\mathcal{L}_{atk}(\theta^*)\|_2 \|\zeta\|_2 \le \epsilon \|\nabla\mathcal{L}_{atk}(\theta^*)\|_2 .
    \end{equation}
    Substituting this back into the quadratic upper bound yields:
    \begin{equation}
        \mathcal{L}_{atk}(\theta^v) \le \mathcal{L}_{atk}(\theta^*) + \epsilon \|\nabla\mathcal{L}_{atk}(\theta^*)\|_2 + \frac{L\epsilon^2}{2}.
    \end{equation}
    The \name\ objective explicitly maximizes the loss within the perturbation ball:
    \begin{equation}
    \begin{aligned}
        \max_{\|\theta^\Delta\| \le \epsilon} \mathcal{L}_{atk}(\theta^* + \theta^\Delta) &\approx \mathcal{L}_{atk}(\theta^*) + \max_{\|\theta^\Delta\| \le \epsilon} \nabla\mathcal{L}_{atk}(\theta^*)^\top \theta^\Delta \\
        &= \mathcal{L}_{atk}(\theta^*) + \epsilon \|\nabla\mathcal{L}_{atk}(\theta^*)\|_2.
    \end{aligned}
    \end{equation}
    Thus, by minimizing the worst-case loss, we are effectively minimizing the upper bound of $\mathcal{L}_{atk}(\theta^v)$.
\end{proof}
Traditional attacks minimize only $\mathcal{L}_{atk}(\theta^*)$. If the landscape is sharp (large second term), $\mathcal{L}_{atk}(\theta^v)$ can still be high.
\name~minimizes both surrogate performance and local sharpness, effectively flattening the loss landscape, thereby reducing the transferability gap.

\subsection{Time Complexity}
As shown in Algorithm~\ref{alg:algorithm}, the computational cost per outer iteration consists of three components: inner min, inner max, and outer min. For inner min, it involves updating surrogate parameters via SGD. The cost is approximately $O(T \cdot |U| \cdot \bar{n} \cdot d)$, where $\bar{n}$ is the average number of interacted items per user and $d$ denotes the embedding dimension.
For inner max, unique to~\name, requires computing gradients to find $\hat{\theta}^\Delta$. The cost is $O(|U| \cdot \bar{n} \cdot d)$, as the cost of calculating $\nabla_\theta \mathcal{L}_{atk}$ is similar to a standard backward pass.
For outer min, it requires one forward and backward pass, costing $O(|U| \cdot \bar{n} \cdot d)$.
Summing these components, the approximate computational complexity for the \name~attack is: $O( L \cdot (T + 2) \cdot |U| \cdot \bar{n} \cdot d )$.
The computational overhead introduced by SharpAP is limited to one additional backward pass per outer iteration compared to the bi-level optimization method.
Given that $T$ dominates the computational cost, this small addition (i.e., $O( L \cdot |U| \cdot \bar{n} \cdot d )$) confirms that \name~incurs marginal time costs while significantly boosting transferability.

%% file: content/5-experiment.tex
\section{Experiments}
In this section, we present extensive experiments on three real-world datasets to evaluate the effectiveness of our proposed \name. We begin by describing the experimental settings, followed by a comparison of the overall performance against state-of-the-art baselines. Finally, we provide a detailed analysis of \name.
\subsection{Experimental Settings}

\begin{table}[t]
    \centering
    \caption{Statistics of the three datasets. “Avg.” represents the average number of items interacted with by each user.}
    \begin{tabular}{c| c| c|c| c| c}
        \hline
        \textbf{Datasets} & \textbf{Users} & \textbf{Items} & \textbf{Ratings} & \textbf{Avg.}& \textbf{Density} \\ \hline
        {MovieLens-1M} & 6,014 & 3,232 & 226,310&38 & 1.17\% \\\hline
        {Gowalla} & 13,149 & 14,007 & 433,356&33 & 0.24\% \\\hline
        Amazon-book & 52,643 & 91,599 & 2,984,108&57 & 0.06\% \\\hline
    \end{tabular}
    
    \label{tab:dataset_stats}
\end{table}

\subsubsection{Datasets}
We select three real-world datasets for the experiments, i.e., MovieLens-1M \cite{movielens1}, Amazon-book \cite{amazon}, and  Gowalla \cite{cho2011friendship}.
To convert MovieLens-1M into an implicit feedback dataset, we follow prior works \cite{wang2020setrank,zhou2018deep}, treating interactions with a rating of 5 as positive feedback, while considering all other interactions as negative feedback.
For Gowalla, we adhere to the data processing procedure outlined in \cite{tang2018personalized,revAdv}. Specifically, we preprocess the raw data by removing cold-start users and items with fewer than 15 interactions. 
For Amazon-book, we follow~\cite{lightgcn_he} and use the 10-core setting to ensure that each user and item have at least 10 interactions.
We adopt a standard training/validation/test split of 7:1:2 across all datasets.
Table~\ref{tab:dataset_stats} summarizes the key statistics of the three datasets.

\begin{table*}[t]
  \centering
  \caption{A comparison of the performance of various attack methods on five representative victim RS, evaluated on \textbf{MovieLens-1M} dataset. The best results are highlighted in bold.}
  \vspace{-0.1cm}
  \label{tab:m1m}
  \begingroup
  \scalebox{0.99}{\begin{tabular}{l| cc| cc| cc| cc| cc}
    \toprule
    \multirow{2}{*}{\textbf{Attacker}}
      & \multicolumn{2}{c}{\textbf{WRMF}}
      & \multicolumn{2}{c}{\textbf{BPR}}
      & \multicolumn{2}{c}{\textbf{LightGCN}}
      & \multicolumn{2}{c}{\textbf{SGL}} & \multicolumn{2}{c}{\textbf{SimGCL}} \\
    \cmidrule(lr){2-3} \cmidrule(lr){4-5} \cmidrule(lr){6-7} \cmidrule(lr){8-9}\cmidrule(lr){10-11}
      & H@20 & N@20
      & H@20 & N@20
      & H@20 & N@20
      & H@20 & N@20& H@20 & N@20 \\
    \midrule
   Clean & 0.0614\tiny±5e-3 & 0.0063\tiny±2e-4 & 0.0294\tiny±3e-3 & 0.0031\tiny±1e-4 & 0.0417\tiny±7e-3 & 0.0044\tiny±3e-4 & 0.0171\tiny±2e-3 & 0.0016\tiny±1e-4&0.0152\tiny±4e-4	&0.0014\tiny±1e-4 \\ \hline
        Random & 0.0620\tiny±7e-3 & 0.0062\tiny±3e-4 & 0.0296\tiny±5e-3 & 0.0032\tiny±3e-4 & 0.0416\tiny±4e-3 & 0.0045\tiny±5e-4 & 0.0175\tiny±7e-3 & 0.0017\tiny±2e-4&0.0159\tiny±5e-4	&0.0014\tiny±2e-4 \\ 
        Popular  & 0.0630\tiny±4e-3 & 0.0063\tiny±2e-4 & 0.0298\tiny±4e-3 & 0.0033\tiny±2e-4 & 0.0418\tiny±2e-3 & 0.0046\tiny±2e-4 & 0.0179\tiny±5e-3 & 0.0018\tiny±2e-4&0.0167\tiny±6e-4	&0.0015\tiny±1e-4 \\ 
        CoVis & 0.0641\tiny±3e-3 & 0.0064\tiny±2e-4 & 0.0305\tiny±2e-3 & 0.0036\tiny±2e-4 & 0.0423\tiny±4e-3 & 0.0051\tiny±3e-4 & 0.0182\tiny±2e-3 & 0.0021\tiny±1e-4&0.0164\tiny±3e-4	&0.0015\tiny±2e-4 \\ 
        PGA & 0.0672\tiny±2e-3 & 0.0069\tiny±3e-4 & 0.0298\tiny±1e-3 & 0.0036\tiny±3e-4 & 0.0434\tiny±5e-3 & 0.0050\tiny±3e-4 & 0.0191\tiny±2e-3 & 0.0020\tiny±2e-4 &0.0180\tiny±5e-4	&0.0017\tiny±3e-4\\ 
        AUSH & 0.0731\tiny±2e-3 & 0.0077\tiny±4e-4 & 0.0315\tiny±2e-3 & 0.0040\tiny±2e-4 & 0.0485\tiny±2e-3 & 0.0062\tiny±2e-4 & 0.0181\tiny±3e-3 & 0.0018\tiny±3e-4&0.0191\tiny±4e-4	&0.0018\tiny±1e-4 \\ \hline
        RevAdv & 0.0743\tiny±4e-3 & 0.0083\tiny±2e-4 & 0.0333\tiny±3e-3 & 0.0043\tiny±2e-4 & 0.0491\tiny±3e-3 & 0.0061\tiny±4e-4 & 0.0196\tiny±1e-3 & 0.0021\tiny±2e-4&0.0203\tiny±5e-4	&0.0019\tiny±2e-4 \\ 
        +SharpAP & \textbf{0.0783\tiny±2e-3} & \textbf{0.0087\tiny±2e-4} & \textbf{0.0434\tiny±2e-3} & \textbf{0.0051\tiny±2e-4} & \textbf{0.0642\tiny±2e-3} & \textbf{0.0074\tiny±4e-4} & \textbf{0.0229\tiny±2e-3} & \textbf{0.0025\tiny±1e-4} &\textbf{0.0234\tiny±5e-4}	&\textbf{0.0022\tiny±1e-4}\\ \hline
        RAPU & 0.0761\tiny±3e-3 & 0.0084\tiny±3e-4 & 0.0307\tiny±5e-3 & 0.0040\tiny±1e-4 & 0.0503\tiny±1e-3 & 0.0070\tiny±1e-4 & 0.0208\tiny±4e-3 & 0.0022\tiny±2e-4 &0.0218\tiny±3e-4	&0.0020\tiny±2e-4\\ 
        +SharpAP & \textbf{0.0808\tiny±1e-3} & \textbf{0.0090\tiny±2e-4} & \textbf{0.0421\tiny±2e-3} & \textbf{0.0046\tiny±2e-4} & \textbf{0.0621\tiny±2e-3} & \textbf{0.0076\tiny±2e-4} & \textbf{0.0225\tiny±4e-3} & \textbf{0.0027\tiny±1e-4} &\textbf{0.0239\tiny±7e-4}	&\textbf{0.0022\tiny±2e-4}\\ \hline
        DADA & 0.0759\tiny±4e-3 & 0.0085\tiny±2e-4 & 0.0344\tiny±3e-3 & 0.0042\tiny±2e-4 & 0.0539\tiny±4e-3 & 0.0068\tiny±2e-4 & 0.0217\tiny±3e-3 & 0.0024\tiny±2e-4 &0.0205\tiny±1e-3	&0.0018\tiny±1e-4\\ 
        +SharpAP & \textbf{0.0798\tiny±3e-3} & \textbf{0.0091\tiny±2e-4} & \textbf{0.0442\tiny±1e-3} & \textbf{0.0052\tiny±1e-4} & \textbf{0.0604\tiny±2e-3} & \textbf{0.0071\tiny±2e-4} & \textbf{0.0240\tiny±3e-3} & \textbf{0.0028\tiny±2e-4} &\textbf{0.0240\tiny±1e-3}	&\textbf{0.0023\tiny±2e-4}\\ \hline
        CLeaR & 0.0741\tiny±2e-3 & 0.0083\tiny±2e-4 & 0.0327\tiny±1e-3 & 0.0042\tiny±2e-4 & 0.0484\tiny±5e-3 & 0.0061\tiny±3e-4 & 0.0198\tiny±1e-3 & 0.0021\tiny±1e-4&0.0217\tiny±8e-4	&0.0020\tiny±2e-4 \\ 
        +SharpAP & \textbf{0.0780\tiny±1e-3} & \textbf{0.0086\tiny±2e-4} & \textbf{0.0439\tiny±2e-3} & \textbf{0.0052\tiny±1e-4} & \textbf{0.0650\tiny±3e-3} & \textbf{0.0075\tiny±1e-4} & \textbf{0.0234\tiny±2e-3} & \textbf{0.0027\tiny±2e-4} &\textbf{0.0232\tiny±5e-4}	&\textbf{0.0022\tiny±1e-4}\\ \hline
        DDSP & 0.0771\tiny±4e-3 & 0.0085\tiny±3e-4 & 0.0339\tiny±5e-3 & 0.0043\tiny±2e-4 & 0.0512\tiny±4e-3 & 0.0070\tiny±1e-4 & 0.0216\tiny±4e-3 & 0.0025\tiny±1e-4&0.0221\tiny±8e-4	&0.0021\tiny±1e-4 \\ 
        +SharpAP & \textbf{0.0803\tiny±4e-3} & \textbf{0.0090\tiny±2e-4} & \textbf{0.0448\tiny±4e-3} & \textbf{0.0054\tiny±3e-4} & \textbf{0.0663\tiny±6e-3} & \textbf{0.0079\tiny±2e-4} & \textbf{0.0242\tiny±3e-3} & \textbf{0.0028\tiny±2e-4}&\textbf{0.0244\tiny±7e-4}	&\textbf{0.0023\tiny±2e-4} \\
    \bottomrule
  \end{tabular}}
  \endgroup
\end{table*}
\begin{table*}[t]
  \centering
  \caption{A comparison of the performance of various attack methods on five representative victim RS, evaluated on \textbf{Gowalla} dataset. The best results are highlighted in bold.}
  \vspace{-0.1cm}
  \label{tab:Gowalla}
  \begingroup
  \scalebox{0.99}{\begin{tabular}{l| cc| cc| cc| cc| cc}
    \toprule
    \multirow{2}{*}{\textbf{Attacker}}
      & \multicolumn{2}{c}{\textbf{WRMF}}
      & \multicolumn{2}{c}{\textbf{BPR}}
      & \multicolumn{2}{c}{\textbf{LightGCN}}
      & \multicolumn{2}{c}{\textbf{SGL}} & \multicolumn{2}{c}{\textbf{SimGCL}} \\
    \cmidrule(lr){2-3} \cmidrule(lr){4-5} \cmidrule(lr){6-7} \cmidrule(lr){8-9}\cmidrule(lr){10-11}
      & H@20 & N@20
      & H@20 & N@20
      & H@20 & N@20
      & H@20 & N@20& H@20 & N@20 \\
    \midrule
   Clean & 0.0119\tiny±5e-4 & 0.0010\tiny±2e-4 & 0.0145\tiny±6e-4 & 0.0016\tiny±1e-4 & 0.0099\tiny±6e-4 & 0.0019\tiny±3e-4 & 0.0097\tiny±5e-4 & 0.0011\tiny±1e-4 & 0.0110\tiny±4e-4 & 0.0009\tiny±1e-4 \\ \hline
        Random & 0.0117\tiny±5e-4 & 0.0009\tiny±1e-4 & 0.0147\tiny±4e-4 & 0.0015\tiny±2e-4 & 0.0115\tiny±9e-4 & 0.0018\tiny±2e-4 & 0.0099\tiny±3e-4 & 0.0012\tiny±2e-4 & 0.0118\tiny±5e-4 & 0.0009\tiny±2e-4 \\ 
        Popular  & 0.0123\tiny±3e-4 & 0.0010\tiny±2e-4 & 0.0149\tiny±3e-4 & 0.0015\tiny±2e-4 & 0.0132\tiny±4e-4 & 0.0019\tiny±2e-4 & 0.0108\tiny±5e-4 & 0.0014\tiny±1e-4 & 0.0123\tiny±3e-4 & 0.0010\tiny±2e-4 \\ 
        CoVis & 0.0125\tiny±4e-4 & 0.0010\tiny±1e-4 & 0.0152\tiny±2e-4 & 0.0017\tiny±1e-4 & 0.0131\tiny±4e-4 & 0.0018\tiny±2e-4 & 0.0112\tiny±4e-4 & 0.0015\tiny±2e-4 & 0.0125\tiny±4e-4 & 0.0010\tiny±1e-4 \\ 
        PGA & 0.0128\tiny±2e-4 & 0.0011\tiny±2e-4 & 0.0156\tiny±5e-4 & 0.0018\tiny±2e-4 & 0.0157\tiny±6e-4 & 0.0021\tiny±3e-4 & 0.0124\tiny±3e-4 & 0.0016\tiny±3e-4 & 0.0132\tiny±5e-4 & 0.0011\tiny±2e-4 \\ 
        AUSH & 0.0131\tiny±3e-4 & 0.0012\tiny±3e-4 & 0.0151\tiny±6e-4 & 0.0017\tiny±2e-4 & 0.0162\tiny±5e-4 & 0.0022\tiny±1e-4 & 0.0107\tiny±2e-4 & 0.0014\tiny±2e-4 & 0.0139\tiny±3e-4 & 0.0012\tiny±2e-4 \\ \hline
        RevAdv & 0.0135\tiny±5e-4 & 0.0013\tiny±2e-4 & 0.0160\tiny±5e-4 & 0.0019\tiny±4e-4 & 0.0165\tiny±4e-4 & 0.0023\tiny±5e-4 & 0.0161\tiny±2e-4 & 0.0019\tiny±3e-4 & 0.0144\tiny±5e-4 & 0.0012\tiny±3e-4 \\ 
        +SharpAP & \textbf{0.0142\tiny±4e-4} & \textbf{0.0018\tiny±2e-4} & \textbf{0.0189\tiny±3e-4} & \textbf{0.0026\tiny±2e-4} & \textbf{0.0213\tiny±1e-4} & \textbf{0.0031\tiny±1e-4} & \textbf{0.0193\tiny±3e-4} & \textbf{0.0025\tiny±2e-4} & \textbf{0.0162\tiny±3e-4} & \textbf{0.0014\tiny±2e-4} \\ \hline
        RAPU & 0.0136\tiny±3e-4 & 0.0014\tiny±3e-4 & 0.0158\tiny±5e-4 & 0.0018\tiny±3e-4 & 0.0164\tiny±5e-4 & 0.0024\tiny±2e-4 & 0.0165\tiny±4e-4 & 0.0019\tiny±2e-4 & 0.0137\tiny±4e-4 & 0.0013\tiny±1e-4 \\ 
        +SharpAP & \textbf{0.0139\tiny±2e-4} & \textbf{0.0019\tiny±1e-4} & \textbf{0.0191\tiny±4e-4} & \textbf{0.0028\tiny±4e-4} & \textbf{0.0210\tiny±5e-4} & \textbf{0.0029\tiny±2e-4} & \textbf{0.0187\tiny±2e-4} & \textbf{0.0021\tiny±1e-4} & \textbf{0.0153\tiny±4e-4} & \textbf{0.0014\tiny±1e-4} \\ \hline
        DADA & 0.0141\tiny±5e-4 & 0.0022\tiny±2e-4 & 0.0173\tiny±3e-4 & 0.0021\tiny±3e-4 & 0.0198\tiny±4e-4 & 0.0027\tiny±3e-4 & 0.0170\tiny±5e-4 & 0.0021\tiny±3e-4 & 0.0140\tiny±3e-4 & 0.0012\tiny±2e-4 \\ 
        +SharpAP & \textbf{0.0163\tiny±4e-4} & \textbf{0.0024\tiny±2e-4} & \textbf{0.0199\tiny±5e-4} & \textbf{0.0030\tiny±2e-4} & \textbf{0.0214\tiny±2e-4} & \textbf{0.0033\tiny±4e-4} & \textbf{0.0185\tiny±4e-4} & \textbf{0.0024\tiny±2e-4} & \textbf{0.0162\tiny±2e-4} & \textbf{0.0014\tiny±1e-4} \\ \hline
        CLeaR & 0.0134\tiny±5e-4 & 0.0013\tiny±3e-4 & 0.0162\tiny±7e-4 & 0.0020\tiny±3e-4 & 0.0174\tiny±6e-4 & 0.0025\tiny±2e-4 & 0.0172\tiny±6e-4 & 0.0020\tiny±3e-4 & 0.0138\tiny±4e-4 & 0.0013\tiny±2e-4 \\ 
        +SharpAP & \textbf{0.0140\tiny±5e-4} & \textbf{0.0016\tiny±2e-4} & \textbf{0.0189\tiny±4e-4} & \textbf{0.0027\tiny±2e-4} & \textbf{0.0195\tiny±5e-4} & \textbf{0.0029\tiny±1e-4} & \textbf{0.0190\tiny±5e-4} & \textbf{0.0025\tiny±1e-4} & \textbf{0.0157\tiny±2e-4} & \textbf{0.0014\tiny±2e-4} \\ \hline
        DDSP & 0.0138\tiny±4e-4 & 0.0017\tiny±1e-4 & 0.0168\tiny±8e-4 & 0.0021\tiny±2e-4 & 0.0183\tiny±6e-4 & 0.0026\tiny±2e-4 & 0.0176\tiny±8e-4 & 0.0021\tiny±4e-4 & 0.0142\tiny±5e-4 & 0.0012\tiny±1e-4 \\ 
        +SharpAP & \textbf{0.0143\tiny±2e-4} & \textbf{0.0018\tiny±1e-4} & \textbf{0.0182\tiny±5e-4} & \textbf{0.0024\tiny±2e-4} & \textbf{0.0197\tiny±5e-4} & \textbf{0.0029\tiny±2e-4} & \textbf{0.0194\tiny±4e-4} & \textbf{0.0024\tiny±2e-4} & \textbf{0.0160\tiny±3e-4} & \textbf{0.0013\tiny±1e-4} \\
    \bottomrule
  \end{tabular}}
  \endgroup
\end{table*}
\subsubsection{Evaluation Metrics}

To quantitatively evaluate the effectiveness of full-user attacks, we follow \cite{dada_nips, incomplete_kdd21} and uniformly sample five items to form the target item set $V^t$. The attack’s success is assessed using the hit ratio (H@K) on real users $U^r$. A hit occurs if at least one target item appears in the Top-K recommendations for any user  $u \in U^r$. Additionally, we employ Normalized Discounted Cumulative Gain (NDCG) to assess the ranking quality of target items within the recommendation list. Specifically, we assign a relevance score of 1 to target items and 0 to all other items for every real user. 
Both metrics are computed for each real test user individually, and then averaged over all real test users.

In group attacks, items are divided into two categories according to their popularity within the attack group (i.e., $U_0$)~\cite{uba_www}. 
Specifically, we select the top 20\% items as popular and the bottom 80\% as unpopular.
Then, we randomly select five popular items and five unpopular items as the target item set $V^t$ from the popular category and the unpopular category, respectively.
Our goal is to attack the $U_0$ group while minimizing the impact on the $U_1$ group. 
To this end, we measure the difference in hit rates for the target items between the two groups, as follows:
\begin{equation}
\begin{gathered}
    \label{eq:dp}
    D@K=H_{U_0}@K-H_{U_1}@K,
\end{gathered}
\end{equation}
where $H_{U_0}@K$ and $H_{U_1}@K$ represent the hit rates of target items for two groups.
A higher $D@K$ value indicates stronger attack performance.

\subsubsection{Baseline Attack Methods}
To evaluate the effectiveness of \name, we select several competing baselines, including heuristic-based, neural network-driven, and gradient-based poisoning attacks.
The details are outlined as follows:
\begin{itemize}[leftmargin=*,topsep=0pt,parsep=0pt]
    \item \textbf{Clean}: Recommendation models are trained on an unattacked dataset.
    \item \textbf{Random Attack}\cite{lam2004shilling}: In this attack, fake users interact with target items as well as a set of randomly selected items.
    \item \textbf{Popular Attack}: According to previous work, \cite{mobasher2007toward,incomplete_kdd21}, each fake user selects 10\% of the most popular items and 90\% of randomly chosen items as the interacted items.
    \item \textbf{CoVis Attack}\cite{yang2017fake}:
    This baseline attack is tailored for association-rule-based recommendation models. In this method, the attacker identifies items for generating fake interactions and injects artificial co-visitations by solving a standard linear programming problem.
    \item \textbf{PGA Attack}\cite{li2016data}: This method targets matrix factorization-based recommendation models. It defines an attack objective and utilizes projected gradient ascent to update the poisoned user’s ratings in order to optimize this objective.
    \item \textbf{AUSH Attack}\cite{aush}: AUSH leverages Generative Adversarial Networks to tailor attacks on recommender systems based on budget and complex goals, such as targeting specific user groups.
    \item \textbf{UBA Attack}\cite{uba_www}: The attack emphasizes the importance of targeting specific users and frames the issue of varying attack difficulty across users using causal language, ultimately calculating the optimal allocation of fake user budgets.
    \item \textbf{RevAdv Attack}\cite{revAdv}: 
    RevAdv is a classic gradient-based attack method, which argues that previous methods calculate the gradient inaccurately. It further proposes a more accurate computation method, achieving state-of-the-art performance.
     \item \textbf{RAPU Attack}\cite{incomplete_kdd21}: 
     Compared to RevAdv, RAPU focuses on situations with incomplete training data and introduces a different attack objective function.
    \item \textbf{DADA Attack}\cite{dada_nips}: DADA introduces a difficulty- and diversity-aware attack objective that ensures easy-to-manipulate users from diverse groups receive more attention, enhancing the overall attack effectiveness.
    \item \textbf{CLeaR Attack}~\cite{CLeaR}: CLeaR employs a dual-objective strategy that promotes a smoother spectral value distribution to broaden user reachability while simultaneously optimizing a rank promotion objective to maximize the exposure of target items.
    \item \textbf{DDSP Attack}~\cite{DDSP}: DDSP employs a dual-promotion objective to simultaneously promote both target items and user-preferred items.
    
\end{itemize}

\subsubsection{Victim Models}
In this section, we carefully select representative RS models as victim models to evaluate the effectiveness of the attack.
\begin{itemize}[leftmargin=*,topsep=0pt,parsep=0pt]
    \item \textbf{WRMF}\cite{wrmf}: It is a foundational and representative factorization-based model for RS using implicit feedback.
    \item \textbf{BPR}\cite{bpr2012}: It is a classic collaborative filtering method that designs a pairwise ranking loss function, which is widely applied in recommendations based on implicit feedback.
    \item \textbf{LightGCN}\cite{lightgcn_he}: It is a state-of-the-art GCN-based method that eliminates feature transformation and nonlinear activation functions in the GCN aggregator.
    \item \textbf{SGL}\cite{sgl}: Compared to LightGCN, SGL further enhances recommendation performance by utilizing self-supervised graph learning.
    \item \textbf{SimGCL}\cite{simgcl}: SimGCL proposes simple graph contrastive learning and noise-based augmentation for graph recommendation.
\end{itemize}

\begin{table*}[t]
  \centering
  \caption{A comparison of the performance of various attack methods on five representative victim RS, evaluated on \textbf{Amazon-book} dataset. The best results are highlighted in bold.}
  \label{tab:book}
  \begingroup
  \scalebox{0.99}{\begin{tabular}{l| cc| cc| cc| cc| cc}
    \toprule
    \multirow{2}{*}{\textbf{Attacker}}
      & \multicolumn{2}{c}{\textbf{WRMF}}
      & \multicolumn{2}{c}{\textbf{BPR}}
      & \multicolumn{2}{c}{\textbf{LightGCN}}
      & \multicolumn{2}{c}{\textbf{SGL}} 
      & \multicolumn{2}{c}{\textbf{SimGCL}} \\
    \cmidrule(lr){2-3} \cmidrule(lr){4-5} \cmidrule(lr){6-7} \cmidrule(lr){8-9}\cmidrule(lr){10-11}
      & H@20(\%) & N@20(\%)
      & H@20(\%) & N@20(\%)
      & H@20(\%) & N@20(\%)
      & H@20(\%) & N@20(\%)
      & H@20(\%) & N@20(\%)\\
    \midrule
   Clean & 0.1615\tiny±2e-3 & 0.0134\tiny±3e-4 & 0.1164\tiny±4e-3 & 0.0101\tiny±9e-4 & 0.1577\tiny±3e-3 & 0.0143\tiny±9e-4 & 0.2090\tiny±1e-3 & 0.0198\tiny±1e-3&0.1623\tiny±4e-3 &0.0158\tiny±2e-4 \\ \hline
        Random & 0.1632\tiny±5e-3 & 0.0138\tiny±1e-3 & 0.1166\tiny±3e-3 & 0.0100\tiny±1e-3 & 0.1592\tiny±3e-3 & 0.0148\tiny±2e-4 & 0.2104\tiny±4e-3 & 0.0199\tiny±6e-4 &0.1744\tiny±5e-3	&0.0165\tiny±4e-4\\ 
        Popular  & 0.1750\tiny±3e-3 & 0.0142\tiny±5e-4 & 0.1190\tiny±1e-3 & 0.0108\tiny±5e-4 & 0.1600\tiny±4e-3 & 0.0151\tiny±5e-4 & 0.2110\tiny±2e-3 & 0.0207\tiny±2e-4 &0.1982\tiny±3e-3	&0.0183\tiny±4e-4\\ 
        CoVis & 0.1741\tiny±4e-3 & 0.0140\tiny±7e-4 & 0.1183\tiny±5e-3 & 0.0105\tiny±3e-4 & 0.1598\tiny±5e-3 & 0.0147\tiny±4e-4 & 0.2106\tiny±6e-3 & 0.0201\tiny±3e-4 &0.1710\tiny±4e-3	&0.0162\tiny±5e-4\\ 
        PGA & 0.1793\tiny±2e-3 & 0.0151\tiny±9e-4 & 0.1222\tiny±5e-3 & 0.0113\tiny±6e-4 & 0.1663\tiny±7e-3 & 0.0156\tiny±8e-4 & 0.2199\tiny±2e-3 & 0.0210\tiny±5e-4 &0.2174\tiny±4e-3	&0.0209\tiny±2e-4\\ 
        AUSH & 0.1752\tiny±5e-3 & 0.0145\tiny±6e-4 & 0.1201\tiny±4e-3 & 0.0110\tiny±3e-4 & 0.1667\tiny±6e-3 & 0.0159\tiny±1e-3 & 0.2137\tiny±3e-3 & 0.0208\tiny±4e-4 &0.1962\tiny±2e-3	&0.0191\tiny±3e-4\\ \hline
        RevAdv & 0.1820\tiny±2e-3 & 0.0150\tiny±5e-4 & 0.1254\tiny±3e-3 & 0.0114\tiny±7e-4 & 0.1702\tiny±2e-3 & 0.0160\tiny±7e-4 & 0.2198\tiny±1e-3 & 0.0211\tiny±2e-4 &0.2144\tiny±4e-3	&0.0205\tiny±4e-4\\ 
        +SharpAP & \textbf{0.1932\tiny±1e-3} & \textbf{0.0168\tiny±2e-4} & \textbf{0.1402\tiny±2e-3} & \textbf{0.0126\tiny±5e-4} & \textbf{0.1882\tiny±1e-3} & \textbf{0.0171\tiny±5e-4} & \textbf{0.2403\tiny±2e-3} & \textbf{0.0227\tiny±4e-4} &\textbf{0.2352\tiny±4e-3}	&\textbf{0.0228\tiny±2e-4}\\ \hline
        RAPU & 0.1843\tiny±3e-3 & 0.0156\tiny±6e-4 & 0.1293\tiny±2e-3 & 0.0119\tiny±6e-4 & 0.1780\tiny±3e-3 & 0.0163\tiny±6e-4 & 0.2204\tiny±3e-3 & 0.0219\tiny±8e-4 &0.2013\tiny±5e-3	&0.0194\tiny±3e-4\\ 
        +SharpAP & \textbf{0.1944\tiny±2e-3} & \textbf{0.0171\tiny±4e-4} & \textbf{0.1440\tiny±1e-3} & \textbf{0.0129\tiny±5e-4} & \textbf{0.1914\tiny±2e-3} & \textbf{0.0178\tiny±6e-4} & \textbf{0.2366\tiny±1e-3} & \textbf{0.0225\tiny±5e-4} &\textbf{0.2454\tiny±2e-3}	&\textbf{0.0233\tiny±2e-4}\\ \hline
        DADA & 0.1905\tiny±5e-3 & 0.0160\tiny±6e-4 & 0.1337\tiny±7e-3 & 0.0123\tiny±1e-3 & 0.1808\tiny±4e-3 & 0.0166\tiny±4e-4 & 0.2189\tiny±2e-3 & 0.0210\tiny±4e-4 &0.1970\tiny±5e-3	&0.0182\tiny±4e-4\\ 
        +SharpAP & \textbf{0.2029\tiny±4e-3} & \textbf{0.0178\tiny±5e-4} & \textbf{0.1497\tiny±2e-3} & \textbf{0.0134\tiny±7e-4} & \textbf{0.2019\tiny±3e-3} & \textbf{0.0181\tiny±8e-4} & \textbf{0.2378\tiny±1e-3} & \textbf{0.0228\tiny±3e-4} &\textbf{0.2268\tiny±3e-3}	&\textbf{0.0215\tiny±1e-4}\\ \hline
        CLeaR & 0.1900\tiny±6e-3 & 0.0158\tiny±7e-4 & 0.1342\tiny±6e-3 & 0.0128\tiny±8e-4 & 0.1812\tiny±5e-3 & 0.0170\tiny±7e-4 & 0.2192\tiny±3e-3 & 0.0218\tiny±9e-4 &0.2107\tiny±2e-3	&0.0204\tiny±5e-4\\ 
        +SharpAP & \textbf{0.2042\tiny±4e-3} & \textbf{0.0180\tiny±4e-4} & \textbf{0.1503\tiny±3e-3} & \textbf{0.0137\tiny±6e-4} & \textbf{0.2024\tiny±4e-3} & \textbf{0.0184\tiny±5e-4} & \textbf{0.2400\tiny±2e-3} & \textbf{0.0229\tiny±5e-4} &\textbf{0.2462\tiny±3e-3}	&\textbf{0.0227\tiny±4e-4}\\ \hline
        DDSP & 0.1913\tiny±2e-3 & 0.0159\tiny±5e-4 & 0.1350\tiny±5e-3 & 0.0129\tiny±5e-4 & 0.1809\tiny±6e-3 & 0.0164\tiny±2e-3 & 0.2203\tiny±5e-3 & 0.0219\tiny±1e-3&0.2245\tiny±4e-3	&0.0212\tiny±4e-4 \\ 
        +SharpAP & \textbf{0.2107\tiny±1e-3} & \textbf{0.0189\tiny±2e-4} & \textbf{0.1594\tiny±4e-3} & \textbf{0.0141\tiny±3e-4} & \textbf{0.2073\tiny±5e-3} & \textbf{0.0190\tiny±2e-3} & \textbf{0.2481\tiny±4e-3} & \textbf{0.0237\tiny±8e-4} &\textbf{0.2530\tiny±2e-3}	&\textbf{0.0241\tiny±3e-4}\\ 
    \bottomrule
  \end{tabular}}
  \endgroup
\end{table*}
\begin{table*}[t]
  \centering
  \caption{A comparison of the performance of various attack methods on five representative victim RS, evaluated on \textbf{MovieLens-1M } dataset. \textbf{LightGCN} is used as the surrogate model. The best results are highlighted in bold.}
  \vspace{-0.1cm}
  \label{tab:ml-1m-gcn}
  \begingroup
  \renewcommand{\arraystretch}{1.2} 
  \scalebox{0.93}{\begin{tabular}{l| cc| cc| cc| cc| cc}
    \toprule
    \multirow{2}{*}{\textbf{Attacker}}
      & \multicolumn{2}{c}{\textbf{WRMF}}
      & \multicolumn{2}{c}{\textbf{BPR}}
      & \multicolumn{2}{c}{\textbf{LightGCN}}
      & \multicolumn{2}{c}{\textbf{SGL}}& \multicolumn{2}{c}{\textbf{SimGCL}} \\
    \cmidrule(lr){2-3} \cmidrule(lr){4-5} \cmidrule(lr){6-7} \cmidrule(lr){8-9}\cmidrule(lr){10-11}
      & H@20 & N@20
      & H@20 & N@20
      & H@20 & N@20
      & H@20 & N@20
      & H@20 & N@20\\
    \midrule
        Clean & 0.0614\tiny±3e-3 & 0.0063\tiny±3e-4 & 0.0294\tiny±1e-3 & 0.0031\tiny±4e-4 & 0.0417\tiny±5e-3 & 0.0044\tiny±2e-4 & 0.0171\tiny±2e-3 & 0.0016\tiny±3e-4&0.0242\tiny±3e-3	&0.0026\tiny±2e-4 \\ \hline
        RevAdv & 0.0703\tiny±2e-3 & 0.0076\tiny±5e-4 & 0.0311\tiny±3e-3 & 0.0040\tiny±5e-4 & 0.0532\tiny±1e-3 & 0.0068\tiny±3e-4 & 0.0182\tiny±2e-3 & 0.0020\tiny±4e-4 &0.0290\tiny±2e-3	&0.0031\tiny±2e-4\\ 
        +SharpAP & \textbf{0.0783\tiny±1e-3} & \textbf{0.0087\tiny±2e-4} & \textbf{0.0434\tiny±4e-3} & \textbf{0.0051\tiny±1e-4} & \textbf{0.0608\tiny±2e-3} & \textbf{0.0071\tiny±2e-4} & \textbf{0.0214\tiny±3e-3} & \textbf{0.0024\tiny±2e-4}&\textbf{0.0317\tiny±6e-4}	&\textbf{0.0033\tiny±1e-4} \\ \hline
        RAPU & 0.0721\tiny±2e-3 & 0.0077\tiny±4e-4 & 0.0301\tiny±4e-3 & 0.0035\tiny±3e-4 & 0.0520\tiny±2e-3 & 0.0066\tiny±2e-4 & 0.0189\tiny±3e-3 & 0.0023\tiny±3e-4&0.0283\tiny±1e-3	&0.0029\tiny±1e-4 \\ 
        +SharpAP & \textbf{0.0778\tiny±1e-3} & \textbf{0.0086\tiny±1e-4} & \textbf{0.0408\tiny±3e-3} & \textbf{0.0044\tiny±2e-4} & \textbf{0.0593\tiny±2e-3} & \textbf{0.0072\tiny±1e-4} & \textbf{0.0230\tiny±3e-3} & \textbf{0.0027\tiny±2e-4}&\textbf{0.0309\tiny±2e-3}	&\textbf{0.0032\tiny±1e-4} \\ \hline
        DADA & 0.0733\tiny±2e-3 & 0.0080\tiny±3e-4 & 0.0336\tiny±2e-3 & 0.0041\tiny±1e-4 & 0.0544\tiny±3e-3 & 0.0069\tiny±3e-4 & 0.0200\tiny±2e-3 & 0.0022\tiny±1e-4&0.0296\tiny±2e-3	&0.0031\tiny±1e-4 \\ 
        +SharpAP & \textbf{0.0764\tiny±2e-3} & \textbf{0.0083\tiny±2e-4} & \textbf{0.0401\tiny±1e-3} & \textbf{0.0048\tiny±1e-4} & \textbf{0.0582\tiny±1e-3} & \textbf{0.0070\tiny±2e-4} & \textbf{0.0252\tiny±2e-3} & \textbf{0.0026\tiny±2e-4} &\textbf{0.0322\tiny±8e-4}	&\textbf{0.0034\tiny±2e-4}\\ \hline
        CLeaR & 0.0727\tiny±2e-3 & 0.0079\tiny±3e-4 & 0.0320\tiny±1e-3 & 0.0041\tiny±1e-4 & 0.0493\tiny±3e-3 & 0.0064\tiny±3e-4 & 0.0183\tiny±2e-3 & 0.0020\tiny±2e-4&0.0277\tiny±2e-3	&0.0029\tiny±1e-4 \\ 
        +SharpAP & \textbf{0.0762\tiny±1e-3} & \textbf{0.0084\tiny±2e-4} & \textbf{0.0427\tiny±2e-3} & \textbf{0.0050\tiny±2e-4} & \textbf{0.0611\tiny±2e-3} & \textbf{0.0072\tiny±2e-4} & \textbf{0.0224\tiny±2e-3} & \textbf{0.0025\tiny±3e-4}&\textbf{0.0308\tiny±1e-3}	&\textbf{0.0032\tiny±2e-4} \\ \hline
        DDSP & 0.0742\tiny±4e-3 & 0.0081\tiny±2e-4 & 0.0330\tiny±3e-3 & 0.0041\tiny±3e-4 & 0.0538\tiny±3e-3 & 0.0069\tiny±2e-4 & 0.0214\tiny±3e-3 & 0.0025\tiny±3e-4&0.0291\tiny±1e-3	&0.0031\tiny±2e-4 \\ 
        +SharpAP & \textbf{0.0788\tiny±2e-3} & \textbf{0.0087\tiny±1e-4} & \textbf{0.0425\tiny±2e-3} & \textbf{0.0051\tiny±2e-4} & \textbf{0.0624\tiny±3e-3} & \textbf{0.0075\tiny±1e-4} & \textbf{0.0235\tiny±2e-3} & \textbf{0.0026\tiny±2e-4}&\textbf{0.0320\tiny±2e-3}	&\textbf{0.0034\tiny±1e-4} \\ 
    \bottomrule
  \end{tabular}}
  \endgroup
\end{table*}
\subsubsection{Implement Details}
For all attacks, without special explanation, we adopt the following attack settings.
The percentage of fake users is fixed to 1\% (i.e., $\delta=1\%$).
The maximum number of items each fake user can interact with, denoted as $N$, is set to the average number of items interacted with by real users in the dataset, as shown in Table.~\ref{tab:dataset_stats}.
Following~\cite{revAdv,dada_nips,fang2020influence}, we employ the representative
factorization-based model WRMF~\cite{wrmf} as the surrogate model and set the learning rate $\lambda_2$ in Alg.\ref{alg:algorithm} to 1.
We also use a more complex model LightGCN~\cite{lightgcn_he} as the surrogate in Table~\ref{tab:ml-1m-gcn}.
For our method \name, we search the perturbation radius $\epsilon$ in Eq.\eqref{eq:bi_31} within \{0.005, 0.02, 0.05, 0.1, 0.2\}.
All experiments are conducted on an NVIDIA A40 GPU with Pytorch-2.1.2.
The reported results are averaged over ten runs with different random strategies (target item sampling, fake user initialization, and optimizer seeds).

\begin{table*}[t]
\centering
    \caption{A comparison of the performance of full-user methods on four representative victim RS, evaluated on \textbf{Movielens-1M} dataset. The target items are selected from the popular and unpopular groups, respectively. The best results are highlighted in bold.}
    \label{tab:pop-ml1m}
\scalebox{0.85}{
\begin{tabular}{c|cccccccc|cccccccc}
\toprule
\multicolumn{1}{c}{\multirow{3}{*}{\textbf{Attacker}}} & \multicolumn{8}{c}{\textbf{Popular item}} & \multicolumn{8}{c}{\textbf{Unpopular item}} \\ \cmidrule(l){2-17} 
\multicolumn{1}{c}{} & \multicolumn{2}{c}{WRMF} & \multicolumn{2}{c}{BPR} & \multicolumn{2}{c}{LightGCN} & \multicolumn{2}{c}{SGL} & \multicolumn{2}{c}{WRMF} & \multicolumn{2}{c}{BPR} & \multicolumn{2}{c}{LightGCN} & \multicolumn{2}{c}{SGL}\\ 
\multicolumn{1}{c}{} & H@20 & N@20 & H@20 & N@20 & H@20 & N@20 & H@20 & N@20& H@20 & N@20& H@20 & N@20 & H@20 & N@20 & H@20 & N@20 \\ \cmidrule(r){1-17}
Clean & 0.0871 & 0.0099 & 0.0467 & 0.0046 & 0.0813 & 0.0086 & 0.0742 & 0.0078 & 0.0026 & 0.0002 & 0.0029 & 0.0003 & 0.0046 & 0.0004 & 0.0035 & 0.0003 \\ \hline
        RevAdv & 0.1107 & 0.0108 & 0.0551 & 0.0058 & 0.1027 & 0.0099 & 0.1035 & 0.0117 & 0.0049 & 0.0006 & 0.0041 & 0.0006 & 0.0078 & 0.0007 & 0.0067 & 0.0005 \\ 
        +SharpAP & \textbf{0.1224} & \textbf{0.0114} & \textbf{0.0762} & \textbf{0.0079} & \textbf{0.1326} & \textbf{0.0141} & \textbf{0.1284} & \textbf{0.0130} & \textbf{0.0053} & \textbf{0.0007} & \textbf{0.0048} & \textbf{0.0007} & \textbf{0.0088} & \textbf{0.0009} & \textbf{0.0079} & \textbf{0.0006} \\ \hline
        RAPU & 0.1202 & 0.0110 & 0.0594 & 0.0062 & 0.1164 & 0.0128 & 0.1090 & 0.0121 & 0.0045 & 0.0004 & 0.0040 & 0.0005 & 0.0073 & 0.0008 & 0.0063 & 0.0005 \\ 
        +SharpAP & \textbf{0.1298} & \textbf{0.0133} & \textbf{0.0787} & \textbf{0.0081} & \textbf{0.1383} & \textbf{0.0150} & \textbf{0.1311} & \textbf{0.0139} & \textbf{0.0051} & \textbf{0.0005} & \textbf{0.0052} & \textbf{0.0006} & \textbf{0.0086} & \textbf{0.0009} & \textbf{0.0081} & \textbf{0.0007} \\ \hline
        DADA & 0.1310 & 0.0128 & 0.0604 & 0.0063 & 0.1105 & 0.0119 & 0.1106 & 0.0124 & 0.0050 & 0.0005 & 0.0047 & 0.0004 & 0.0074 & 0.0007 & 0.0069 & 0.0006 \\ 
        +SharpAP & \textbf{0.1423} & \textbf{0.0151} & \textbf{0.0831} & \textbf{0.0084} & \textbf{0.1351} & \textbf{0.0142} & \textbf{0.1287} & \textbf{0.0132} & \textbf{0.0062} & \textbf{0.0007} & \textbf{0.0058} & \textbf{0.0005} & \textbf{0.0084} & \textbf{0.0008} & \textbf{0.0078} & \textbf{0.0008} \\ \hline
        CLeaR & 0.1302 & 0.0124 & 0.0611 & 0.0060 & 0.1047 & 0.0118 & 0.1124 & 0.0128 & 0.0048 & 0.0004 & 0.0048 & 0.0004 & 0.0076 & 0.0006 & 0.0067 & 0.0007 \\ 
        +SharpAP & \textbf{0.1399} & \textbf{0.0147} & \textbf{0.0752} & \textbf{0.0078} & \textbf{0.1266} & \textbf{0.0134} & \textbf{0.1306} & \textbf{0.0138} & \textbf{0.0053} & \textbf{0.0005} & \textbf{0.0055} & \textbf{0.0005} & \textbf{0.0088} & \textbf{0.0007} & \textbf{0.0076} & \textbf{0.0008} \\ \hline
        DDSP & 0.1374 & 0.0138 & 0.0632 & 0.0065 & 0.1180 & 0.0123 & 0.1202 & 0.0130 & 0.0050 & 0.0005 & 0.0051 & 0.0005 & 0.0072 & 0.0007 & 0.0080 & 0.0007 \\ 
        +SharpAP & \textbf{0.1433} & \textbf{0.0159} & \textbf{0.0840} & \textbf{0.0085} & \textbf{0.1372} & \textbf{0.0150} & \textbf{0.1327} & \textbf{0.0142} & \textbf{0.0058} & \textbf{0.0006} & \textbf{0.0062} & \textbf{0.0006} & \textbf{0.0078} & \textbf{0.0008} & \textbf{0.0085} & \textbf{0.0008} \\ 
 \bottomrule
\end{tabular}
}
\end{table*}
\begin{table*}[t]

\centering
    \caption{A comparison of the performance of group attack methods on three representative victim RS, evaluated on \textbf{Movielens-1M} dataset. The best results are highlighted in bold.}
    \label{tab:group}
\scalebox{0.84}{
\begin{tabular}{c|cccccccc|cccccccc}
\toprule
\multicolumn{1}{c}{\multirow{3}{*}{\textbf{Attacker}}} & \multicolumn{8}{c}{\textbf{Popular item}} & \multicolumn{8}{c}{\textbf{Unpopular item}} \\ \cmidrule(l){2-17} 
\multicolumn{1}{c}{} & \multicolumn{2}{c}{WRMF} & \multicolumn{2}{c}{BPR} & \multicolumn{2}{c}{LightGCN} & \multicolumn{2}{c}{SGL} & \multicolumn{2}{c}{WRMF} & \multicolumn{2}{c}{BPR} & \multicolumn{2}{c}{LightGCN} & \multicolumn{2}{c}{SGL}\\ 
\multicolumn{1}{c}{} & D@10 & D@20 & D@10 & D@20 & D@10 & D@20 & D@10 & D@20& D@10 & D@20& D@10 & D@20 & D@10 & D@20 & D@10 & D@20 \\ \cmidrule(r){1-17}
Clean & 0.0269 & 0.0322 & 0.0112 & 0.0177 & 0.0351 & 0.0578&0.0313&0.0495 & 0.0002 & -0.0046 & -0.0009 & -0.0028 & 0.0000 & 0.0009 &0.0001&0.0005\\ \hline
AUSH & 0.0344 & 0.0461 & 0.0138 & 0.0203 & 0.0425 & 0.0712&0.0387&0.0531 & -0.0003 & -0.005 & -0.0006 & -0.0017 & 0.0008 & 0.0012 &0.0003&0.0009\\
UBA & 0.0435 & 0.0519 & 0.0176 & 0.0324 & 0.0472 & 0.0800&0.0424&0.069 & 0.0026 & 0.0019 & 0.0034 & 0.0026 & 0.0043 & 0.0062 &0.0035&0.0058\\ \hline
RevAdv & 0.0412 & 0.0501 & 0.0153 & 0.0231 & 0.0498 & 0.0816&0.0412&0.0633 & 0.0014 & 0.0005 & 0.0025 & 0.0012 & 0.0039 & 0.0056 &0.0026&0.0041\\
+SharpAP & \textbf{0.0537} & \textbf{0.062} & \textbf{0.0362} & \textbf{0.0578} & \textbf{0.0626} & \textbf{0.1082} &\bfseries0.057&\bfseries0.0897& \textbf{0.0105} & \textbf{0.0062} & \textbf{0.0068} & \textbf{0.0031} & \textbf{0.0075} & \textbf{0.0131} &\bfseries0.0058&\bfseries0.0103\\ \hline
RAPU & 0.0395 & 0.0486 & 0.0164 & 0.0269 & 0.0511 & 0.0937&0.0445&0.0782 & 0.0012 & 0.0004 & 0.0031 & 0.0028 & 0.0042 & 0.0060 &0.0034&0.0055\\ 
+SharpAP & \textbf{0.0485} & \textbf{0.0597} & \textbf{0.0337} & \textbf{0.0546} & \textbf{0.0723} & \textbf{0.1108}&\bfseries0.0671&\bfseries0.1015 & \textbf{0.0113} & \textbf{0.0076} & \textbf{0.0070} & \textbf{0.0035} & \textbf{0.0084} & \textbf{0.0153} &\bfseries0.0061&\bfseries0.0093\\ \hline
DADA & 0.0483 & 0.0551 & 0.0266 & 0.0347 & 0.0481 & 0.0807&0.0386&0.0558 & 0.0027 & 0.0012 & 0.0043 & 0.0030 & 0.0054 & 0.0088 &0.004&0.0067\\ 
+SharpAP & \textbf{0.0631} & \textbf{0.0727} & \textbf{0.0421} & \textbf{0.0634} & \textbf{0.0712} & \textbf{0.1085}&\bfseries0.0614&\bfseries0.0972 & \textbf{0.0095} & \textbf{0.0068} & \textbf{0.0074} & \textbf{0.0042} & \textbf{0.0096} & \textbf{0.0146} &\bfseries0.0063&\bfseries0.0105\\  \hline
CLeaR & 0.0420 & 0.0514 & 0.0221 & 0.0305 & 0.0494 & 0.0828&0.0426&0.0758 & 0.0024 & 0.0009 & 0.0038 & 0.0029 & 0.0046 & 0.0071 &0.0042&0.0070\\  
+SharpAP & \textbf{0.0588} & \textbf{0.0690} & \textbf{0.0417} & \textbf{0.0603} & \textbf{0.0736} & \textbf{0.1093}&\bfseries0.0677&\bfseries0.1103 & \textbf{0.0086} & \textbf{0.0052} & \textbf{0.0069} & \textbf{0.0038} & \textbf{0.0090} & \textbf{0.0122} &\bfseries0.0060&\bfseries0.0101\\ \hline
DDSP & 0.0497 & 0.0560 & 0.0277 & 0.0338 & 0.0460 & 0.0775&0.0451&0.0776 & 0.0022 & 0.0007 & 0.0031 & 0.0024 & 0.0040 & 0.0069 &0.0044&0.0073\\  
+SharpAP & \textbf{0.0612} & \textbf{0.0709} & \textbf{0.0398} & \textbf{0.0587} & \textbf{0.0680} & \textbf{0.0993}&\bfseries0.0692&\bfseries0.1286 & \textbf{0.0083} & \textbf{0.0046} & \textbf{0.0059} & \textbf{0.0036} & \textbf{0.0030} & \textbf{0.0108} &\bfseries0.0065&\bfseries0.0122\\ 
\bottomrule
\end{tabular}
}
\vspace{-0.3cm}
\end{table*}

\subsection{Overall Performance}
\subsubsection{Full-user Attack Performance}
We report the full-user attack performance of all methods under different Top-K settings from Tables~\ref{tab:m1m},~\ref{tab:Gowalla},~\ref{tab:book},~\ref{tab:ml-1m-gcn}, and~\ref{tab:pop-ml1m}, and have the following observations:
\begin{itemize}[leftmargin=*,topsep=0pt,parsep=0pt]
    \item All attack methods increase the exposure probability of target items. However, heuristic-based methods (e.g., Random, Popular, and CoVis) do not achieve impressive performance because they do not directly optimize the attack objective. Instead, they manually select fake user interaction records to enhance the co-occurrence between target items and other items.
    \item Although some gradient-based methods achieve good attack performance by directly optimizing the attack objective, they exhibit overfitting.
    Our \name~mitigates the overfitting issue and improves performance even on models beyond the surrogate model.
    Note that when using RevAdv as the base model, SharpAP boosts HR@20 on BPR and WRMF by 48\% and 27\%, respectively. We hypothesize that the pairwise objective intrinsically creates a loss landscape with higher curvature (sharper local minima) and greater sensitivity to input perturbations compared to WRMF.
    \item As shown in Table~\ref{tab:ml-1m-gcn}, when using LightGCN as the surrogate without SharpAP, the attack performs exceptionally well against a LightGCN victim, which is intuitive. However, transferability to other victims (e.g., WRMF, BPR) degrades compared to using the WRMF surrogate. We attribute this to LightGCN's complex graph aggregation mechanism, which causes the generated poisoned data to overfit specific high-order graph artifacts.
    \item Table~\ref{tab:pop-ml1m} presents a comparison of attack performance when the target items are selected from popular (Top 20\%) and unpopular (Bottom 80\%) items, respectively. The H@20/N@20 for unpopular items is consistently lower. Attacking unpopular items is harder than attacking popular items across all methods. This is because unpopular items have sparse interaction data, resulting in less stable embeddings that are located farther from real users in the latent space~\cite{revAdv,aush}.
    \item 
    Finally, as demonstrated in these Tables, our proposed \name~consistently achieves superior attack performance compared to all baselines across three benchmark datasets under varying victim models.
    The core advantage of \name~lies in its ability to enhance poisoning robustness through poisoning the worst-case model instead of a fixed surrogate model. 
    This ability effectively alleviates the overfitting of poisoned data to the surrogate model, thereby enhancing the transferability of the attack.
    
\end{itemize}

\subsubsection{Group Attack Performance}
As mentioned in Sec.\ref{sub:full_and_group}, we can use the group attack objective to assist \name~in performing group attacks.
Here, we compare the group attack performance of different methods in Table.~\ref{tab:group}. 
Since the MovieLens-1M dataset includes user profile information, we adopt it as the evaluation dataset and divide users into two groups based on gender: a female group ($|U_0| = 1{,}705$) and a male group ($|U_1| = 4{,}309$). Users with missing gender information are excluded.
Our findings are as follows:
\begin{itemize}[leftmargin=*,topsep=0pt,parsep=0pt]
    \item All methods perform worse on unpopular target items compared to popular ones, meaning that cold items are harder to promote. This is likely because cold items are farther from real users in the latent space, making the attack more challenging.
    Similar findings have also been reported in \cite{revAdv,aush}.
    \item Our proposed \textit{SharpAP} consistently outperforms all baselines on popular and unpopular items. 
    Specifically, \textit{SharpAP} improves RevAdv by 30\% in terms of D@10 for popular items and 600\% for unpopular items in WRMF, respectively.
    \item On all three backbones (RevAdv, RAPU, and DADA) equipped with our method, we observed a significant performance advantage, even on unpopular items, which validates the effectiveness of our method.
\end{itemize}

\begin{table}[t]
\centering
\caption{All-user attack performance under three defense models, evaluated on Movielens-1M dataset using H@10.}
    \label{tab:defense1}
\scalebox{0.75}{\begin{tabular}{c|cccc|cccc} \hline
\multirow{2}{*}{Attacker} & \multicolumn{4}{c}{BPR} & \multicolumn{4}{c}{LightGCN} \\ 
 & Attack & PCA & APT&PamaCF & Attack & PCA & APT &PamaCF\\ \hline
RevAdv & 0.0125 & 0.0121 & 0.0113 & 0.0119&   0.0186 & 0.0184 & 0.0170&0.0167 \\
+SharpAP & 0.0181 & 0.0162 & 0.0154 &0.0146 &  0.0259 & 0.0241 & 0.0253&0.0216 \\\hline
RAPU & 0.0121 & 0.0118 & 0.0115 &0.0117&  0.0191 & 0.0186 & 0.0179&0.0160 \\
+SharpAP & 0.0177 & 0.0173 & 0.0159 & 0.0147&  0.0261 & 0.0255 & 0.0247&0.0223 \\\hline
DADA & 0.0138 & 0.0120 & 0.0124 & 0.0118&  0.0211 & 0.0198 & 0.0195&0.0188 \\
+SharpAP & 0.0194 & 0.0187 & 0.0176 &0.0173 &  0.0248 & 0.0233 & 0.0219&0.0204 \\\hline
\end{tabular}}
\end{table}

\begin{table}[t]
\centering
\caption{Group attack performance under three defense models, evaluated on Movielens-1M dataset using D@10.}
    \label{tab:defense2}
\scalebox{0.75}{\begin{tabular}{c|cccc|cccc} \hline
\multirow{2}{*}{Attacker} & \multicolumn{4}{c}{BPR} & \multicolumn{4}{c}{LightGCN} \\ 
 & Attack & PCA & APT&PamaCF & Attack & PCA & APT &PamaCF \\\hline
RevAdv & 0.0153 & 0.0136 & 0.0128&0.0123   & 0.0498 & 0.0463 & 0.0458&0.0445  \\
+SharpAP & 0.0362 & 0.0254 & 0.0279 &0.0234     & 0.0626 & 0.0591 & 0.0574&0.0522  \\\hline
RAPU & 0.0164 & 0.0143 & 0.0136  &0.0128        & 0.0511 & 0.0416 & 0.0453&0.0390  \\
+SharpAP & 0.0337 & 0.0271 & 0.0224 &0.0207     & 0.0723 & 0.0665 & 0.0582&0.0564  \\\hline
DADA & 0.0266 & 0.0253 & 0.0248 &0.0211         & 0.0481 & 0.0440 & 0.0421&0.0413  \\
+SharpAP & 0.0421 & 0.0322 & 0.0345 &0.0304     & 0.0712 & 0.0593 & 0.0574&0.0558  \\\hline
\end{tabular}}
\vspace{-0.4cm}
\end{table}

\begin{figure}[t]
  \begin{center}
    \includegraphics[scale=0.43]{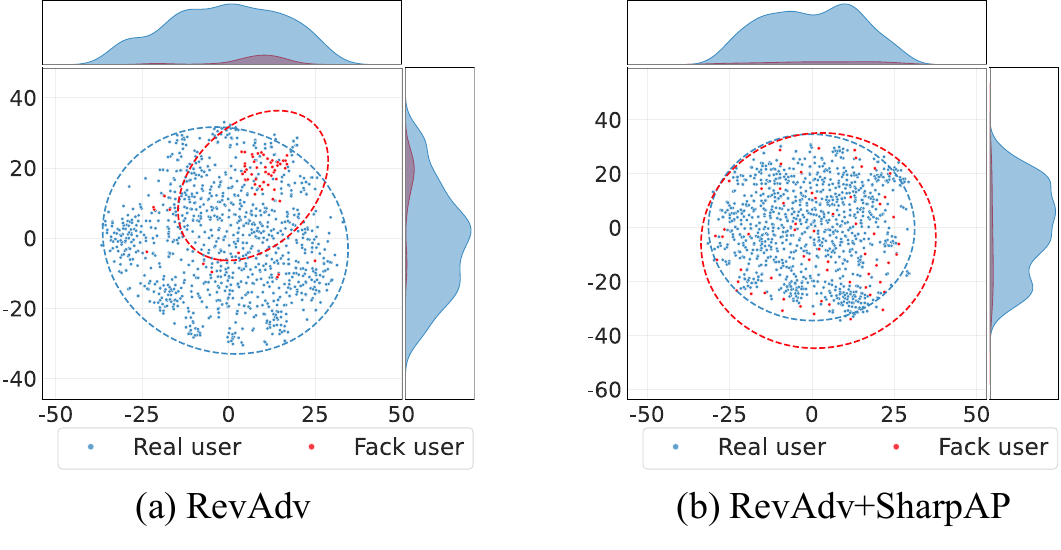}
  \end{center}
  \vspace{-0.40cm}
  \caption{Visualization of user representations with t-SNE. The top and right curve graphs display the marginal distributions for two reduced dimensions.}
  \label{fig:case}
\end{figure}

\begin{figure}[t]
\centering
\includegraphics[width=1.0\columnwidth]{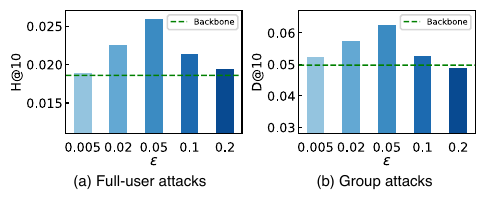}
\vspace{-20pt}
\caption{Performance comparisons under different perturbation radius $\epsilon$ using RevAdv as the backbone.}
\label{fig:perturbation}
\vspace{-0.3cm}
\end{figure}

\subsection{Performance on Defense Models}
To further validate the robustness of the proposed method, we evaluate full-user attacks and group attacks using existing defense models.
We select two representative models, BPR and LightGCN, as victim models. The results on SGL and SimGCL are similar, we do not present them due to space limitations.
We first evaluate the effectiveness of existing defense models against full-user attacks. 
We examine three representative defense models: PCA \cite{pca}, APT \cite{gray1}, and PamaCF~\cite{PamaCF}. PCA detects and removes fake users from the training data, while APT injects fake users to enhance the model’s robustness. 
PamaCF defends against poisoning attacks by dynamically assigning personalized adversarial perturbation magnitudes based on each user's embedding scale.
Table~\ref{tab:defense1} presents the results of three representative attack methods and our \name, evaluated on the MovieLens-1M dataset.
From Table \ref{tab:defense1}, we observe the following findings: 1) All defenses reduce the performance of all attackers, demonstrating their effectiveness in defending against attacks. 
2) However, even with the defense models in place, \name~consistently achieves higher hit ratios than the baseline attackers.
To understand \name's evasion capabilities, we conduct a visualization case study on the representative LightGCN victim model.
Specifically, we randomly select 1000 real users and all fake users (i.e., 60) in the MovieLens-1M dataset, respectively.
Fig.~\ref{fig:case} shows that while RevAdv's fake users form tight clusters, \name's are dispersed among real users. This demonstrates that \name's curvature-aware optimization effectively minimizes statistical footprints.
We also conduct experiments against group attacks, and the results are shown in Table \ref{tab:defense2}. The target items are randomly selected from the popular group. 
We can observe that all defense methods reduce the $D@10$ value, indicating effective resistance against attacks. 
The defense results prove the robustness of \name, i.e., the sharpness-aware attack can generate more robust poisoned data.

\begin{small}
\begin{figure*} [t]
  \begin{center}
  \includegraphics[scale=0.75]{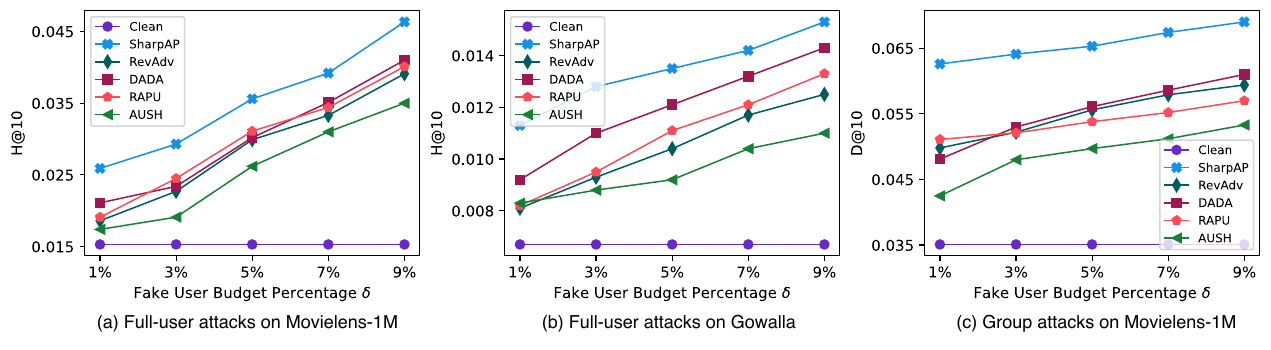}
  \end{center}
    \vspace{-0.40cm}
  \caption{
  The attack performance under different fake user budget percentages. $\delta$ is the percentage of fake users relative to the number of real users $U^r$.} 
  \label{fig:budget}
\end{figure*}
\end{small}
\begin{small}
\begin{figure*} [t]
  \begin{center}
  \includegraphics[scale=0.75]{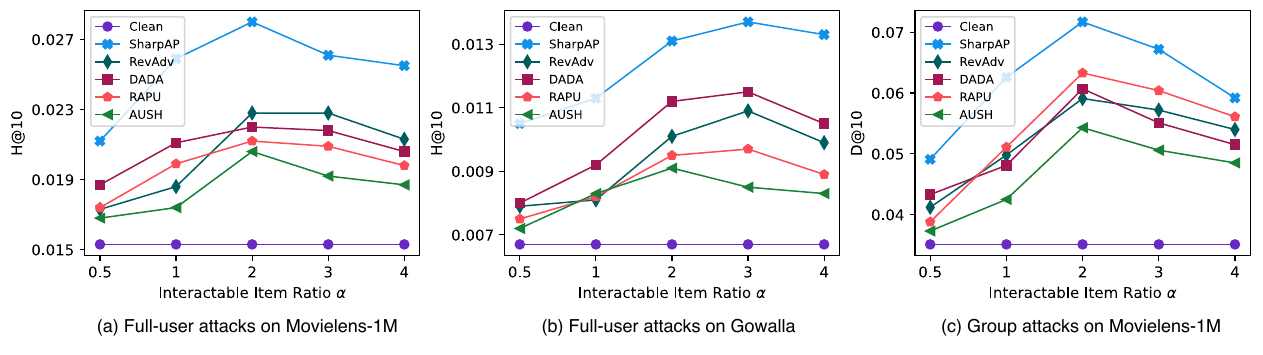}
  \end{center}
    \vspace{-0.40cm}
  \caption{
  The attack performance under different ratios of interactable items. $\alpha$ denotes the scaling ratio relative to the average number of items interacted with by real users.} 
  \label{fig:items}
  \vspace{-0.4cm}
\end{figure*}
\end{small}

\subsection{Parameter Sensitivity}

\subsubsection{Perturbation Radius $\epsilon$}
The perturbation radius $\epsilon$ in Eq.~\eqref{eq:bi_31} is a critical hyperparameter in our proposed method, defining the search space for the worst-case model. In all experiments, we set $\epsilon = 0.05$. To further analyze its effect on attack performance, we explore different values from the set \{0.005, 0.02, 0.05, 0.1, 0.2\}. We evaluate the effect of the perturbation radius $\epsilon$ on both full-user and group attacks, using RevAdv as the backbone and LightGCN as the victim model. The experiments are conducted on the MovieLens-1M dataset, and results are presented in Fig.~\ref{fig:perturbation}.
We have several observations from Fig.~\ref{fig:perturbation}.
First, a smaller $\epsilon$ causes the model to regress to the backbone attack due to the limited search space.
Second, with the increase of $\epsilon$, this implies the enhancement of perturbation space, \name~can explore a larger space to find the worst-case model, which achieves a better attack performance. 
However, a larger $\epsilon$ may cause the model to deviate from the training objective of the recommendation system, resulting in a poor attack performance.
Finally, \name~achieves the best attack results with a moderate value of perturbation radius $\epsilon$.

\subsubsection{Fake User Budget Percentage $\delta$}
In this section, we investigate the effect of fake user budget percentage $\delta$ in the attackers' capability constraint $|U^f|\le \delta|U^r|$. 
We systematically vary $\delta$ across \{1\%, 3\%, 5\%, 7\%, 9\%\} and conduct evaluations on multiple attack scenarios using LightGCN as the victim model. To ensure clarity, we restrict our comparison to representative and powerful baselines: RevAdv, DADA, RAPU, and AUSH.
As shown in Fig.~\ref{fig:budget}(a) and (b), attack effectiveness exhibits a positive correlation with fake user budget percentage $\delta$ in full-user attack scenarios across both MovieLens-1M and Gowalla datasets.
Notably, our method consistently outperforms all baselines, as evidenced by the superior positioning of its performance curve (blue). A critical observation reveals that performance gains scale more prominently on MovieLens-1M than on Gowalla with increasing $\delta$.
This phenomenon can be attributed to the inherent characteristics of the datasets, as MovieLens-1M has only 3,232 items whereas Gowalla contains 14,007 items, making it more difficult to promote target items to the Top-10 recommendation list.
Fig.~\ref{fig:budget}(c) presents group attack performance on MovieLens-1M. Consistent with previous findings, all attackers demonstrate enhanced performance with $\delta$ increasing.
Meanwhile, our proposed \name~achieves the best attack performance. These two attack scenarios further validate the robustness of our method against varying fake user budget percentages.

\subsubsection{Maximum number of interactable items $N$}
Following prior works~\cite{uba_www,dada_nips,incomplete_kdd21}, we constrain the number of items that each fake user can interact with to be no greater than the average number of items interacted with by real users., i.e., $||\mathbf{R}^f[u^f]||_0 \le N$ for each fake user $u^f$.
To further investigate the effect of behavioral patterns, we systematically modulate the interaction capacity using scaling ratios $\alpha \in \{0.5, 1, 2, 3, 4\}$, yielding adjusted interaction thresholds $\alpha N$.
Fig.~\ref{fig:items} reveals a non-monotonic relationship between the number of interactions and attack performance in both full-user and group attack scenarios. 
Specifically, performance initially improves as interaction capacity increases, reaching an optimal threshold, after which it begins to decline.
This suggests that an excessive number of interactable items may introduce noisy signals, ultimately weakening the effectiveness of the attack.
Furthermore, compared with these powerful baselines, our proposed method demonstrates robust superiority under different maximum numbers of items.

\begin{figure}[t]
\centering
\includegraphics[width=0.90\columnwidth]{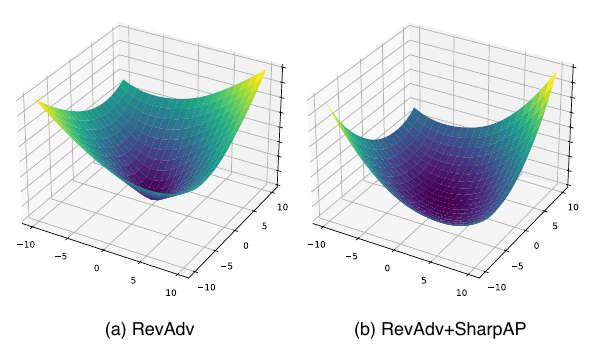}
\vspace{-10pt}
\caption{Visualization of loss landscape on Movielens-1M trained with or without SharpAP.}
\label{fig:loss}
\vspace{-0.2cm}
\end{figure}

\subsection{Visualization of Loss Landscape}
To validate that our method effectively reduces the sharpness of the attack loss landscape, we visualize the full-user attack loss landscape using the methodology from~\cite{li2018visualizing}. 
First, we fix the surrogate model parameters after training as $\theta^*$.  
Then, we generate two random perturbation vectors $w_1, w_2 \in \mathbb{R}^d$ (where $d = \dim(\theta^*)$) from the standard normal distribution.
Next, we sample 20 equidistant values for $m$ and $n$ within $[-10, 10]$, forming a 20×20 grid. For each grid point $(m_i, n_j)$, compute the perturbed parameters $\theta_{ij} = \theta^* + m_i w_1 + n_j w_2$, simulating the shift of model.
Finally, we calculate the attack loss $\mathcal{L}_{atk}(\theta_{ij}; \mathbf{R}^r)$ for each $\theta_{ij}$.  
As shown in Fig.~\ref{fig:loss}, our method produces a markedly smoother loss landscape compared to the backbone. This smoothness indicates that the attack remains effective even when victim model parameters deviate from $\theta^*$, corroborating our sharpness-aware design's ability to enhance the transferability.

\subsection{Running Time Comparison}
In this section, we report the running time of our proposed method on the MovieLens-1M dataset. All experiments were conducted on a server equipped with a single NVIDIA A40 GPU. As summarized in Table~\ref{tab:time}, the results indicate that \name~incurs a marginal computational overhead (approximately 5\%) compared to the baseline methods. We argue that this modest increase in time is acceptable, given that \name~significantly enhances attack transferability.

%% file: content/2-related.tex
\section{Related Work}
\label{related}
\subsection{Recommender Systems}
Modern recommender systems have shown remarkable performance in enhancing user experience by effectively aligning the diverse preferences of individual users with a wide range of available items~\cite{zhang2019deep,yang2023generative}. These systems predominantly utilize CF to generate recommendations. Given a user-item rating matrix, item-based CF approaches focus on calculating the similarity in item behavior, leveraging these similarity scores to make personalized recommendations~\cite{jin2004automatic,sarwar2001item}.
Subsequently, matrix factorization-based CF models have become increasingly prevalent due to their ability to better capture nuanced user preferences and deliver more accurate, personalized recommendations. A central challenge for these matrix factorization models is the learning of high-quality embeddings, which are crucial for optimizing recommendation performance.
For instance, WRMF~\cite{wrmf} is a well-established latent factor model that applies matrix factorization to derive user and item embeddings. Similarly, BPR~\cite{bpr2012} introduces a pairwise ranking loss function, which has been widely adopted in recommendation systems that rely on implicit feedback.
Since user-item interactions inherently form a bipartite graph, researchers have proposed neural graph-based models to capture higher-order collaboration signals in the learned embeddings. A notable example is LightGCN~\cite{lightgcn_he}, which updates the embeddings of users and items iteratively by aggregating neighborhood embeddings from previous layers to encode these higher-order relationships.
Additionally, to enhance the accuracy and robustness of GCNs for recommendation, SGL~\cite{sgl} augments the classical supervised recommendation task with an auxiliary self-supervised task. This task enhances node representation learning by maximizing the agreement between different views of the same node.
As novel recommendation models continue to emerge, they present significant challenges to existing poisoning methods, particularly those that generate poisoned data based on a fixed surrogate model.

\begin{table}[t]
\centering
\caption{Running time (s) on MovieLens-1M.}
\label{tab:time}
\begin{tabular}{l|ccccc}
\toprule
 & RevAdv & RAPU & DADA & CLeaR & DDSP \\
\midrule
Baseline & 465.22 & 459.48 & 523.74 & 438.62 & 372.34 \\
+SharpAP & 483.09 & 480.37 & 538.16 & 455.04 & 388.91 \\
\bottomrule
\end{tabular}
\vspace{-0.2cm}
\end{table}
\subsection{Injective Attacks}
From an attacker's perspective, poisoning attacks are designed to manipulate RS. While untargeted attacks aim to erode overall recommendation quality, targeted attacks seek to promote or demote specific items within distinct user groups (i.e., group attacks) or across all users (i.e., full-user attacks)~\cite{wang2024poisoning}. This paper focuses on targeted attacks, which are the most extensively studied category in RS~\cite{DDSP,CLeaR,dada_nips}.
Attackers can easily introduce bias into RS by injecting some fake users (i.e., injective attacks).
Existing methods on injective attacks against RS can be broadly categorized into three paradigms: heuristic-based, neural network-driven, and gradient-based attacks.
Heuristic-based attacks leverage the insight that similar users tend to share similar interests, and they rely on manually crafted fake profiles.
For example, in a random attack\cite{lam2004shilling}, fake users interact with target items as well as a set of randomly selected items.
Popular attack \cite{mobasher2007toward,incomplete_kdd21} not only gives high ratings to the target item but also to several popular items, thereby enhancing the attack’s effectiveness.
However, heuristic attacks are unable to account for all recommendation patterns, leading to limited performance.
Neural network-driven attacks propose neural networks to learn probability distributions of selected items for each fake user.
Specifically, these methods aim to assign high scores to target items for fake users while ensuring that the generated fake profiles closely resemble real users as much as possible.
For example, AUSH~\cite{aush} employs a tailored GAN network to generate fake user profiles.
However, LegUP~\cite{legup} critiques AUSH for employing an indirect generation loss that is only loosely connected to the reconstruction of selected user data. This design choice may limit the overall effectiveness of the attack. To address this limitation, LegUP enhances AUSH by integrating a surrogate model, thereby further improving the poisoning performance.
Gradient-based attacks relax the discrete fake user behaviors into continuous values and directly optimize by maximizing the attack objective.
For example, RevAdv~\cite{revAdv} uses a more accurate gradient calculation method.
DADA~\cite{dada_nips} proposes a difficulty and diversity-aware objective function, which maximizes the attack performance.
CLeaR~\cite{CLeaR} employs a dual-objective strategy that promotes a smoother spectral value distribution to broaden user reachability while simultaneously optimizing a rank promotion objective.
DDSP~\cite{DDSP} employs a dual-promotion objective to simultaneously promote both target items and user-preferred items, thereby ensuring attack stealthiness.
Although these methods achieve good performance, they fail to explicitly model transferability.
Table~\ref{tab:related} summarizes the existing attacks.

\begin{table}[t]
\centering
\caption{Summary of existing attackers.}
\label{tab:related}
\scalebox{0.9}{\begin{tabular}{lcc}
\toprule
\textbf{Attack Paradigm} & \textbf{Attacker} & \textbf{Mechanism for Transferability} \\
\midrule
\multirow{2}{*}{Neural network--based} 
& AUSH~\cite{aush}   & GAN \\
& LegUP~\cite{legup}  & GAN \\
\midrule
\multirow{5}{*}{Gradient--based} 
& RevAdv~\cite{revAdv}         & Improved gradient calculation \\ 
& DADA~\cite{dada_nips}           & Difficulty and diversity aware \\
& CLeaR~\cite{CLeaR}          & Dispersion and rank promotion \\
& DDSP~\cite{DDSP}           & Diversity aware dual--promotion \\
& SharpAP (ours) & Sharpness--aware minimization \\
\bottomrule
\end{tabular}}
\vspace{-0.4cm}
\end{table}
\subsection{Sharpness-Aware Minimization}

Sharpness-aware minimization~\cite{wen2023sharpness,keskar2017large} is a highly effective regularization technique that enhances the model's generalization across various settings. 
A lower sharpness value is generally associated with better generalization performance.
Specifically, sharpness-aware minimization achieves improved generalization by minimizing the maximum loss within a neighborhood of the current parameter, rather than optimizing the loss at a single point. This strategy yields solutions that are more robust to small parameter perturbations, thereby enhancing both generalization and robustness~\cite{li2024friendly}.
Building on this idea, recent studies \cite{sharpness_mini} and \cite{zheng2021regularizing} independently propose minimizing the loss in the direction of the worst-case perturbation from the current parameter to improve generalization. Similarly, the work \cite{wu2020adversarial} introduces a nearly identical method aimed at improving the robust generalization of adversarial training. In addition, ASAM~\cite{kwon2021asam} dynamically adjusts the perturbation region according to the scale of the model weights.
ImbSAM~\cite{zhou2023imbsam} extends the applicability of sharpness-aware minimization to scenarios with highly imbalanced data distributions, effectively addressing the trade-off between sharpness minimization and data imbalance.
Recently, in computer vision, the work~\cite{he2024sharpness} experimentally demonstrated that using the sharpness-aware principle can improve attack transferability when the perturbation is continuous. However, directly applying sharpness awareness in recommender systems faces challenges such as optimizing over discrete data and the lack of theoretical justification.

%% file: content/6-conclusion.tex
\section{CONCLUSION}

In this paper, we propose \name, a novel method to enhance the cross-model transferability of injective attacks on recommender systems. 
Specifically, we argue that existing methods for generating poisoned data based on a fixed surrogate model fundamentally rely on a precarious assumption.
This reliance causes the poisoned data to overfit the surrogate model, resulting in poor transferability when encountering model structure shifts. 
To address this issue, we introduce the sharpness-aware minimization principle to seek the approximately worst-case model during the attack process. 
By iteratively optimizing poisoned data against the worst-case model instead of a fixed surrogate model, we enhance transferability across various victim models.
Our method can be formulated as a sharpness-aware min-max-min tri-level optimization problem, where the maximization performs a bounded perturbation to seek the worst-case model. 
Extensive experiments on three real-world datasets across representative attacks demonstrated the effectiveness of the proposed \name.
In the future, we will explore defense methods against sharpness-aware poisoning.

%% file: author.tex
\begin{IEEEbiography}[{\includegraphics[width=1in,height=1.25in,clip,keepaspectratio]{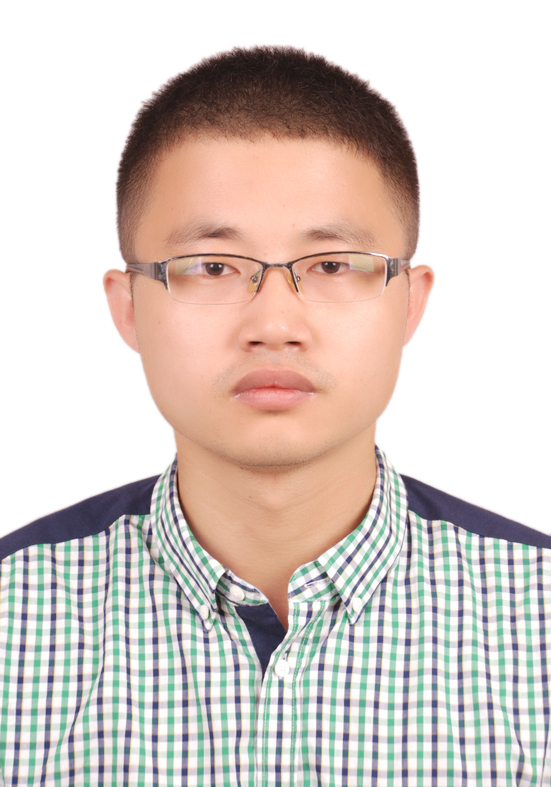}}]
{Junsong Xie} is currently pursuing a Ph.D. degree at Hefei University of Technology (HFUT), China. He received the master's degree from the University of Science and Technology of China (USTC). He has published several papers in referred conferences and journals, such as IJCAI and Frontiers of Computer Science. His major research interest lies on data mining and recommender systems.
\end{IEEEbiography}

\begin{IEEEbiography}[{\includegraphics[width=1in,height=1.25in,clip,keepaspectratio]{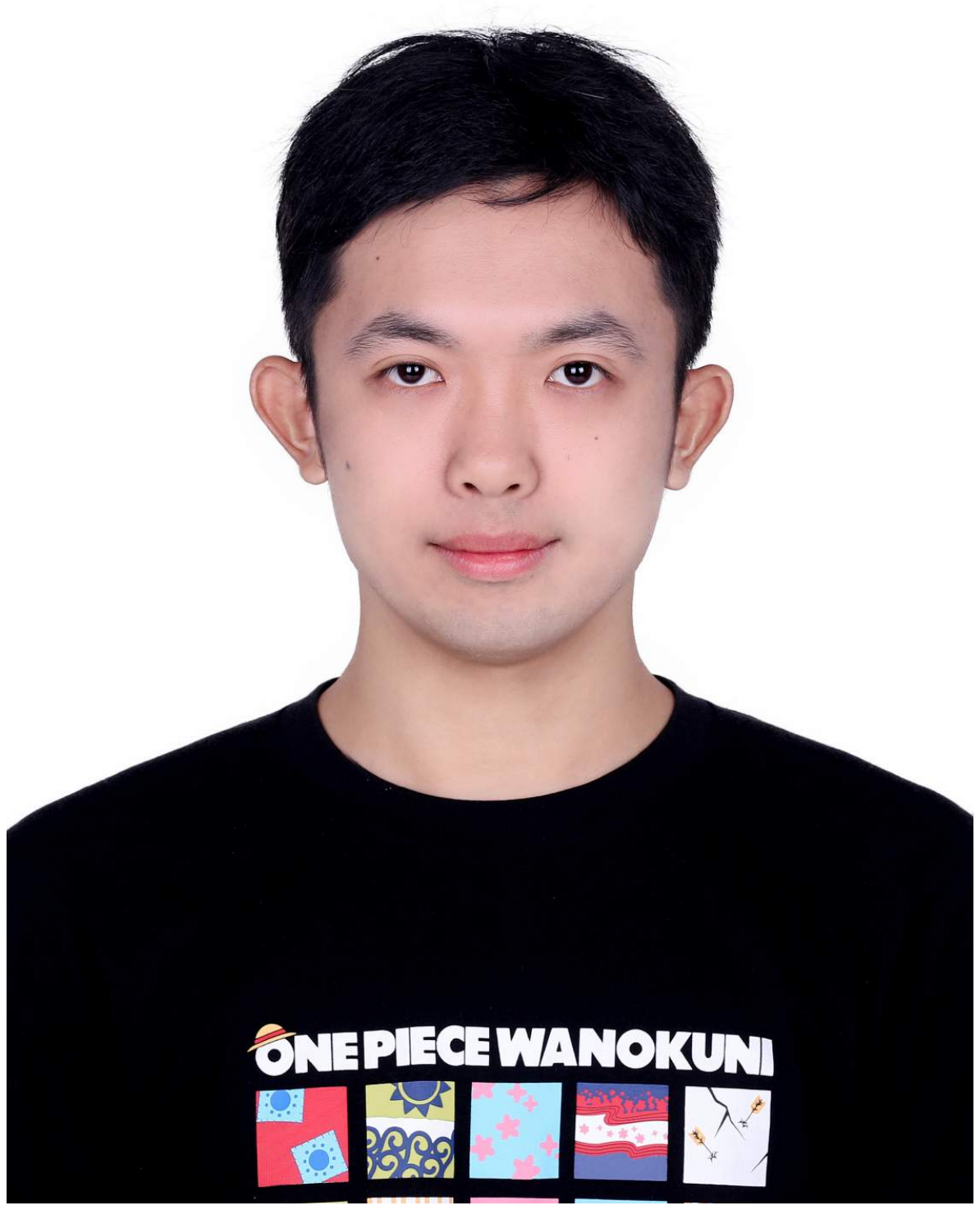}}]
{Yonghui Yang} is currently a Research Fellow at the National University of Singapore. He received the Ph.D degree from the 
Hefei University of Technology, China. He has published over 10 papers in referred journals and conferences, such as IEEE TKDE, TBD, KDD, SIGIR, IJCAI, and ACM Multimedia. His research interests include data-centric recommendation and LLM safety.
\end{IEEEbiography}

\begin{IEEEbiography}[{\includegraphics[width=1in,height=1.25in,clip,keepaspectratio]{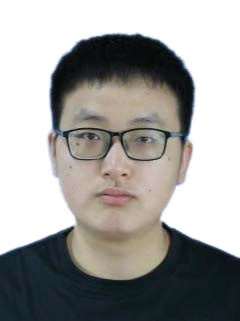}}]
{Pengyang Shao} is currently pursuing a PhD degree at Hefei University of Technology (HFUT), China. He received his Bachelor’s degree in 2019 from the same university. His research interest lies on data mining, and large language models. He has published several papers in leading conferences and journals, including KDD, WWW, ACM TOIS and SCIS.
\end{IEEEbiography}

\begin{IEEEbiography}[{\includegraphics[width=1in,height=1.25in,clip,keepaspectratio]{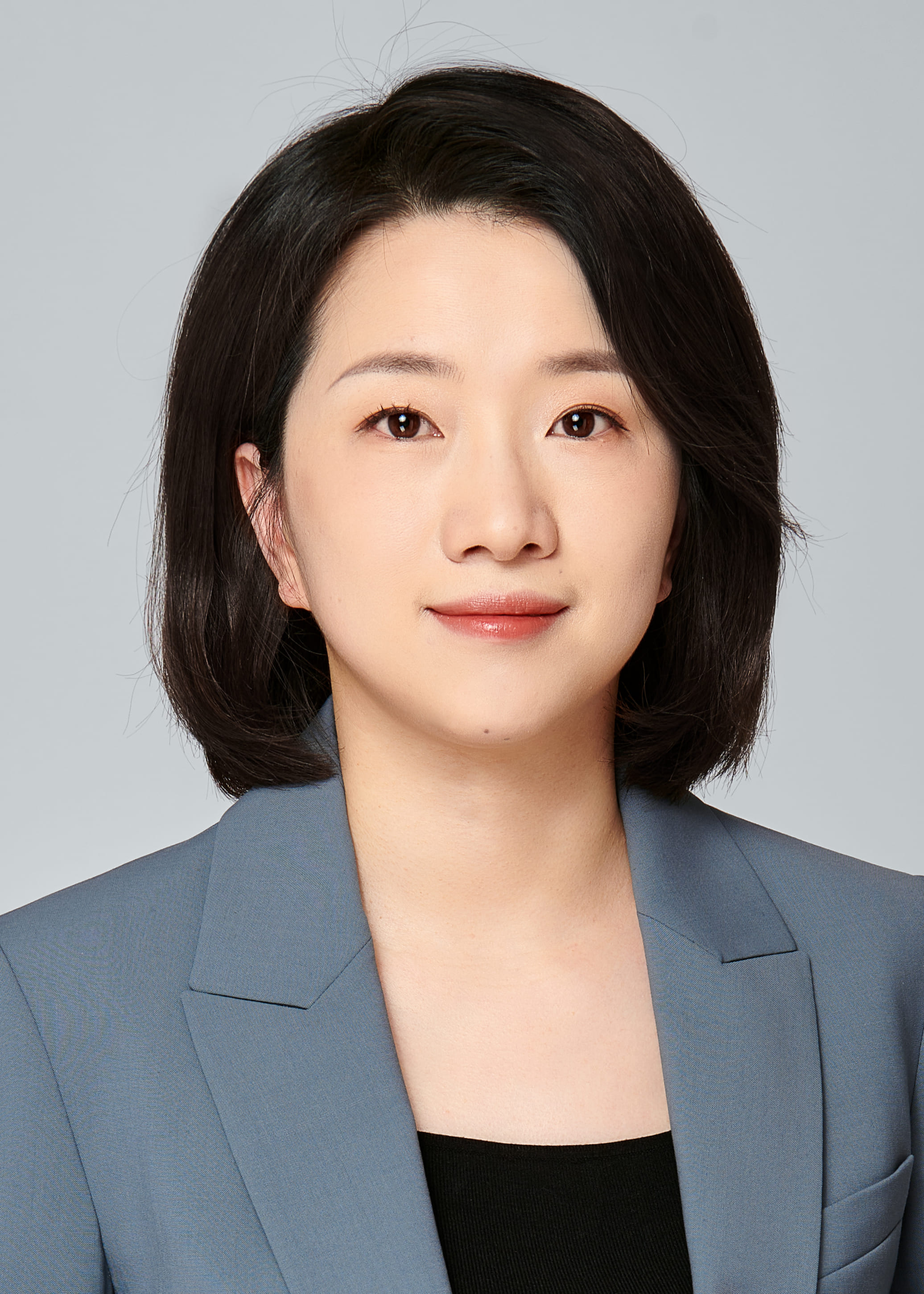}}]
{Le Wu}  is currently a professor at the Hefei University of Technology (HFUT), China. She received her Ph.D. degree from the University of Science and Technology of China (USTC). Her general area of research interests are data mining and knowledge engineering, personalized recommendation,  trustworthy user modeling and applications. She has published more than 70 papers in leading journals and conferences, such as TKDE, TOIS, WWW, SIGIR, KDD, NeurIPS and so on. She is an associate editor of IEEE Trans. on Big Data, AI Open and Frontieres of Computer Science.
\end{IEEEbiography}

%% file: ref.bib
@article{wule2022survey,
  title={A survey on accuracy-oriented neural recommendation: From collaborative filtering to information-rich recommendation},
  author={Wu, Le and He, Xiangnan and Wang, Xiang and Zhang, Kun and Wang, Meng},
  journal={IEEE TKDE},
  volume={35},
  number={5},
  pages={4425--4445},
  year={2022},
  publisher={IEEE}
}

@article{wu2018collaborative,
  title={Collaborative neural social recommendation},
  author={Wu, Le and Sun, Peijie and Hong, Richang and Ge, Yong and Wang, Meng},
  journal={IEEE transactions on systems, man, and cybernetics: systems},
  volume={51},
  number={1},
  pages={464--476},
  year={2018},
  publisher={IEEE}
}

@article{zhang2019deep,
  title={Deep learning based recommender system: A survey and new perspectives},
  author={Zhang, Shuai and Yao, Lina and Sun, Aixin and Tay, Yi},
  journal={ACM computing surveys (CSUR)},
  volume={52},
  number={1},
  pages={1--38},
  year={2019},
  publisher={ACM New York, NY, USA}
}

@inproceedings{yang2023generative,
  title={Generative-contrastive graph learning for recommendation},
  author={Yang, Yonghui and Wu, Zhengwei and Wu, Le and Zhang, Kun and Hong, Richang and Zhang, Zhiqiang and Zhou, Jun and Wang, Meng},
  booktitle={SIGIR},
  pages={1117--1126},
  year={2023}
}

@article{wang2024poisoning,
  title={Poisoning Attacks and Defenses in Recommender Systems: A Survey},
  author={Wang, Zongwei and Yu, Junliang and Gao, Min and Yuan, Wei and Ye, Guanhua and Sadiq, Shazia and Yin, Hongzhi},
  journal={arXiv preprint arXiv:2406.01022},
  year={2024}
}

@inproceedings{revAdv,
  title={Revisiting adversarially learned injection attacks against recommender systems},
  author={Tang, Jiaxi and Wen, Hongyi and Wang, Ke},
  booktitle={RecSys},
  pages={318--327},
  year={2020}
}

@inproceedings{uba_www,
  title={Uplift Modeling for Target User Attacks on Recommender Systems},
  author={Wang, Wenjie and Wang, Changsheng and Feng, Fuli and Shi, Wentao and Ding, Daizong and Chua, Tat-Seng},
  booktitle={WWW},
  pages={3343--3354},
  year={2024}
}

@article{dada_nips,
  title={Revisiting injective attacks on recommender systems},
  author={Li, Haoyang and Di, Shimin and Chen, Lei},
  journal={NeurIPS},
  volume={35},
  pages={29989--30002},
  year={2022}
}

@inproceedings{zhang2024improving,
  title={Improving the shortest plank: Vulnerability-aware adversarial training for robust recommender system},
  author={Zhang, Kaike and Cao, Qi and Wu, Yunfan and Sun, Fei and Shen, Huawei and Cheng, Xueqi},
  booktitle={RecSys},
  pages={680--689},
  year={2024}
}

@article{legup,
  title={Shilling black-box recommender systems by learning to generate fake user profiles},
  author={Lin, Chen and Chen, Si and Zeng, Meifang and Zhang, Sheng and Gao, Min and Li, Hui},
  journal={IEEE TNNLS},
  volume={35},
  number={1},
  pages={1305--1319},
  year={2022},
  publisher={IEEE}
}

@inproceedings{aush,
  title={Attacking recommender systems with augmented user profiles},
  author={Lin, Chen and Chen, Si and Li, Hui and Xiao, Yanghua and Li, Lianyun and Yang, Qian},
  booktitle={CIKM},
  pages={855--864},
  year={2020}
}

@inproceedings{incomplete_kdd21,
  title={Data poisoning attack against recommender system using incomplete and perturbed data},
  author={Zhang, Hengtong and Tian, Changxin and Li, Yaliang and Su, Lu and Yang, Nan and Zhao, Wayne Xin and Gao, Jing},
  booktitle={KDD},
  pages={2154--2164},
  year={2021}
}

@inproceedings{fang2020influence,
  title={Influence function based data poisoning attacks to top-n recommender systems},
  author={Fang, Minghong and Gong, Neil Zhenqiang and Liu, Jia},
  booktitle={WWW},
  pages={3019--3025},
  year={2020}
}

@inproceedings{guo2023targeted,
  title={Targeted shilling attacks on gnn-based recommender systems},
  author={Guo, Sihan and Bai, Ting and Deng, Weihong},
  booktitle={CIKM},
  pages={649--658},
  year={2023}
}

@inproceedings{burke2005limited,
  title={Limited knowledge shilling attacks in collaborative filtering systems},
  author={Burke, Robin and Mobasher, Bamshad and Bhaumik, Runa},
  booktitle={IJCAI},
  pages={17--24},
  year={2005}
}

@inproceedings{yang2017fake,
  title={Fake Co-visitation Injection Attacks to Recommender Systems.},
  author={Yang, Guolei and Gong, Neil Zhenqiang and Cai, Ying},
  booktitle={NDSS},
  year={2017}
}

@inproceedings{fan2021attacking,
  title={Attacking black-box recommendations via copying cross-domain user profiles},
  author={Fan, Wenqi and Derr, Tyler and Zhao, Xiangyu and Ma, Yao and Liu, Hui and Wang, Jianping and Tang, Jiliang and Li, Qing},
  booktitle={ICDE},
  pages={1583--1594},
  year={2021},
  organization={IEEE}
}

@inproceedings{wu2021triple,
  title={Triple adversarial learning for influence based poisoning attack in recommender systems},
  author={Wu, Chenwang and Lian, Defu and Ge, Yong and Zhu, Zhihao and Chen, Enhong},
  booktitle={KDD},
  pages={1830--1840},
  year={2021}
}

@article{cheng2024towards,
  title={Towards Robust Recommendation: A Review and an Adversarial Robustness Evaluation Library},
  author={Cheng, Lei and Huang, Xiaowen and Sang, Jitao and Yu, Jian},
  journal={arXiv preprint arXiv:2404.17844},
  year={2024}
}

@article{wrmf,
  title={Matrix factorization techniques for recommender systems},
  author={Koren, Yehuda and Bell, Robert and Volinsky, Chris},
  journal={Computer},
  volume={42},
  number={8},
  pages={30--37},
  year={2009},
  publisher={IEEE}
}

@article{movielens1,
  title={The movielens datasets: History and context},
  author={Harper, F Maxwell and Konstan, Joseph A},
  journal={TIIS},
  volume={5},
  number={4},
  pages={1--19},
  year={2015},
  publisher={Acm New York, NY, USA}
}

@inproceedings{bpr2012,
  title={BPR: Bayesian personalized ranking from implicit feedback},
  author={Rendle, Steffen and Freudenthaler, Christoph and Gantner, Zeno and Schmidt-Thieme, Lars},
  booktitle={UAI},
  pages={452--461},
  year={2009}
}

@inproceedings{lightgcn_he,
  title={Lightgcn: Simplifying and powering graph convolution network for recommendation},
  author={He, Xiangnan and Deng, Kuan and Wang, Xiang and Li, Yan and Zhang, Yongdong and Wang, Meng},
  booktitle={SIGIR},
  pages={639--648},
  year={2020}
}

@inproceedings{sharpness_mini,
  title={Sharpness-aware Minimization for Efficiently Improving Generalization},
  author={Foret, Pierre and Kleiner, Ariel and Mobahi, Hossein and Neyshabur, Behnam},
  booktitle={ICLR},
  year={2021}
}

@article{li2016data,
  title={Data poisoning attacks on factorization-based collaborative filtering},
  author={Li, Bo and Wang, Yining and Singh, Aarti and Vorobeychik, Yevgeniy},
  journal={NeurIPS},
  volume={29},
  year={2016}
}

@inproceedings{gray1,
  title={Fight fire with fire: towards robust recommender systems via adversarial poisoning training},
  author={Wu, Chenwang and Lian, Defu and Ge, Yong and Zhu, Zhihao and Chen, Enhong and Yuan, Senchao},
  booktitle={SIGIR},
  pages={1074--1083},
  year={2021}
}

@inproceedings{cho2011friendship,
  title={Friendship and mobility: user movement in location-based social networks},
  author={Cho, Eunjoon and Myers, Seth A and Leskovec, Jure},
  booktitle={KDD},
  pages={1082--1090},
  year={2011}
}

@inproceedings{wang2020setrank,
  title={Setrank: A setwise bayesian approach for collaborative ranking from implicit feedback},
  author={Wang, Chao and Zhu, Hengshu and Zhu, Chen and Qin, Chuan and Xiong, Hui},
  booktitle={AAAI},
  volume={34},
  number={04},
  pages={6127--6136},
  year={2020}
}

@inproceedings{zhou2018deep,
  title={Deep interest network for click-through rate prediction},
  author={Zhou, Guorui and Zhu, Xiaoqiang and Song, Chenru and Fan, Ying and Zhu, Han and Ma, Xiao and Yan, Yanghui and Jin, Junqi and Li, Han and Gai, Kun},
  booktitle={KDD},
  pages={1059--1068},
  year={2018}
}

@inproceedings{tang2018personalized,
  title={Personalized top-n sequential recommendation via convolutional sequence embedding},
  author={Tang, Jiaxi and Wang, Ke},
  booktitle={WSDM},
  pages={565--573},
  year={2018}
}

@inproceedings{lam2004shilling,
  title={Shilling recommender systems for fun and profit},
  author={Lam, Shyong K and Riedl, John},
  booktitle={Proceedings of the 13th international conference on World Wide Web},
  pages={393--402},
  year={2004}
}

@article{mobasher2007toward,
  title={Toward trustworthy recommender systems: An analysis of attack models and algorithm robustness},
  author={Mobasher, Bamshad and Burke, Robin and Bhaumik, Runa and Williams, Chad},
  journal={ACM Transactions on Internet Technology (TOIT)},
  volume={7},
  number={4},
  pages={23--es},
  year={2007},
  publisher={ACM New York, NY, USA}
}

@article{pca,
  title={Unsupervised strategies for shilling detection and robust collaborative filtering},
  author={Mehta, Bhaskar and Nejdl, Wolfgang},
  journal={User Modeling and User-Adapted Interaction},
  volume={19},
  pages={65--97},
  year={2009},
  publisher={Springer}
}

@inproceedings{wen2023sharpness,
  title={How Sharpness-Aware Minimization Minimizes Sharpness?},
  author={Wen, Kaiyue and Ma, Tengyu and Li, Zhiyuan},
  booktitle={ICLR},
  year={2023}
}

@inproceedings{zheng2021regularizing,
  title={Regularizing neural networks via adversarial model perturbation},
  author={Zheng, Yaowei and Zhang, Richong and Mao, Yongyi},
  booktitle={CVPR},
  pages={8156--8165},
  year={2021}
}

@inproceedings{wu2020adversarial,
  title={Adversarial weight perturbation helps robust generalization},
  author={Wu, Dongxian and Xia, Shu-Tao and Wang, Yisen},
  booktitle={NeurIPS},
  pages={2958--2969},
  year={2020}
}

@inproceedings{kwon2021asam,
  title={Asam: Adaptive sharpness-aware minimization for scale-invariant learning of deep neural networks},
  author={Kwon, Jungmin and Kim, Jeongseop and Park, Hyunseo and Choi, In Kwon},
  booktitle={ICML},
  pages={5905--5914},
  year={2021},
  organization={PMLR}
}

@inproceedings{jin2004automatic,
  title={An automatic weighting scheme for collaborative filtering},
  author={Jin, Rong and Chai, Joyce Y and Si, Luo},
  booktitle={SIGIR},
  pages={337--344},
  year={2004}
}

@inproceedings{sarwar2001item,
  title={Item-based collaborative filtering recommendation algorithms},
  author={Sarwar, Badrul and Karypis, George and Konstan, Joseph and Riedl, John},
  booktitle={Proceedings of the 10th international conference on World Wide Web},
  pages={285--295},
  year={2001}
}

@inproceedings{chen2023does,
  title={Why does sharpness-aware minimization generalize better than SGD?},
  author={Chen, Zixiang and Zhang, Junkai and Kou, Yiwen and Chen, Xiangning and Hsieh, Cho-Jui and Gu, Quanquan},
  booktitle={NeurIPS},
  pages={72325--72376},
  year={2023}
}

@article{li2018visualizing,
  title={Visualizing the loss landscape of neural nets},
  author={Li, Hao and Xu, Zheng and Taylor, Gavin and Studer, Christoph and Goldstein, Tom},
  journal={Advances in neural information processing systems},
  volume={31},
  year={2018}
}

@inproceedings{sgl,
  title={Self-supervised graph learning for recommendation},
  author={Wu, Jiancan and Wang, Xiang and Feng, Fuli and He, Xiangnan and Chen, Liang and Lian, Jianxun and Xie, Xing},
  booktitle={SIGIR},
  pages={726--735},
  year={2021}
}

@article{wzw02,
  title={Gray-box shilling attack: An adversarial learning approach},
  author={Wang, Zongwei and Gao, Min and Li, Jundong and Zhang, Junwei and Zhong, Jiang},
  journal={TIST},
  volume={13},
  number={5},
  pages={1--21},
  year={2022},
  publisher={ACM New York, NY}
}

@inproceedings{li2024friendly,
  title={Friendly sharpness-aware minimization},
  author={Li, Tao and Zhou, Pan and He, Zhengbao and Cheng, Xinwen and Huang, Xiaolin},
  booktitle={CVPR},
  pages={5631--5640},
  year={2024}
}

@inproceedings{zhou2023imbsam,
  title={Imbsam: A closer look at sharpness-aware minimization in class-imbalanced recognition},
  author={Zhou, Yixuan and Qu, Yi and Xu, Xing and Shen, Hengtao},
  booktitle={CVPR},
  pages={11345--11355},
  year={2023}
}

@inproceedings{keskar2017large,
  title={On Large-Batch Training for Deep Learning: Generalization Gap and Sharp Minima},
  author={Keskar, Nitish Shirish and Mudigere, Dheevatsa and Nocedal, Jorge and Smelyanskiy, Mikhail and Tang, Ping Tak Peter},
  booktitle={ICLR},
  year={2017}
}

@inproceedings{amazon,
  title={Image-based recommendations on styles and substitutes},
  author={McAuley, Julian and Targett, Christopher and Shi, Qinfeng and Van Den Hengel, Anton},
  booktitle={SIGIR},
  pages={43--52},
  year={2015}
}

@inproceedings{simgcl,
  title={Are graph augmentations necessary? simple graph contrastive learning for recommendation},
  author={Yu, Junliang and Yin, Hongzhi and Xia, Xin and Chen, Tong and Cui, Lizhen and Nguyen, Quoc Viet Hung},
  booktitle={SIGIR},
  pages={1294--1303},
  year={2022}
}

@inproceedings{CLeaR,
  title={Unveiling vulnerabilities of contrastive recommender systems to poisoning attacks},
  author={Wang, Zongwei and Yu, Junliang and Gao, Min and Yin, Hongzhi and Cui, Bin and Sadiq, Shazia},
  booktitle={KDD},
  pages={3311--3322},
  year={2024}
}

@inproceedings{DDSP,
  title={Diversity-aware Dual-promotion Poisoning Attack on Sequential Recommendation},
  author={Zhao, Yuchuan and Chen, Tong and Yu, Junliang and Zheng, Kai and Cui, Lizhen and Yin, Hongzhi},
  booktitle={SIGIR},
  pages={1634--1644},
  year={2025}
}

@inproceedings{PamaCF,
  title={Understanding and improving adversarial collaborative filtering for robust recommendation},
  author={Zhang, Kaike and Cao, Qi and Wu, Yunfan and Sun, Fei and Shen, Huawei and Cheng, Xueqi},
  booktitle={NeurIPS},
  volume={37},
  pages={120381--120417},
  year={2024}
}

@inproceedings{he2024sharpness,
  title={Sharpness-Aware Data Poisoning Attack},
  author={He, Pengfei and Xu, Han and Ren, Jie and Cui, Yingqian and Zeng, Shenglai and Liu, Hui and Aggarwal, Charu and Tang, Jiliang},
  booktitle={ICLR},
  year={2024}
}

@article{wang2025graph,
  title={When Graph Contrastive Learning Backfires: Spectral Vulnerability and Defense in Recommendation},
  author={Wang, Zongwei and Gao, Min and Yu, Junliang and Sadiq, Shazia and Yin, Hongzhi and Liu, Ling},
  journal={ACM Transactions on Information Systems},
  year={2025},
  publisher={ACM New York, NY}
}

@inproceedings{wang2025id,
  title={Id-free not risk-free: Llm-powered agents unveil risks in id-free recommender systems},
  author={Wang, Zongwei and Gao, Min and Yu, Junliang and Gao, Xinyi and Nguyen, Quoc Viet Hung and Sadiq, Shazia and Yin, Hongzhi},
  booktitle={Proceedings of the 48th International ACM SIGIR Conference on Research and Development in Information Retrieval},
  pages={1902--1911},
  year={2025}
}
